\pdfoutput=1
\documentclass[letterpaper]{article}

\PassOptionsToPackage{compress}{natbib}
\PassOptionsToPackage{hyperfootnotes=false}{hyperref}
\usepackage[final]{corl_2026}
\usepackage[utf8]{inputenc}
\usepackage[T1]{fontenc}
\usepackage{hyperref}
\usepackage{url}
\usepackage{booktabs}
\usepackage{amsfonts}
\usepackage{amsmath}
\usepackage{amssymb}
\usepackage{mathtools}
\usepackage{nicefrac}
\usepackage{microtype}
\usepackage{graphicx}
\usepackage{xcolor}
\usepackage{enumitem}
\usepackage{algorithm}
\usepackage{algpseudocode}
\usepackage{bm}
\usepackage{multirow}
\usepackage[normalem]{ulem}
\usepackage{caption}
\usepackage{wrapfig}

\hypersetup{
  pdftitle={Constraint-Decomposed Impedance Control from Generative Policy for Articulated Object Manipulation},
  pdfauthor={Kyungseo Park, Byeongdo Lim, Jungbin Lim, Minseok Choi, Frank C. Park}
}

\newcommand{\se}{\mathrm{se}}
\newcommand{\SE}{\mathrm{SE}}
\newcommand{\so}{\mathrm{so}}
\newcommand{\SO}{\mathrm{SO}}
\newcommand{\Ad}{\mathrm{Ad}}
\newcommand{\Rbb}{\mathbb{R}}

\title{Constraint-Decomposed Impedance Control from Generative Policy for Articulated Object Manipulation}

\author{%
  \makebox[0pt]{Kyungseo Park\setcounter{footnote}{1}\thanks{Current affiliation: Department of Mechanical Engineering, KAIST. Email: \texttt{erickun0125@kaist.ac.kr}} \quad Byeongdo Lim \quad Jungbin Lim \quad Minseok Choi \quad Frank C. Park} \\[3pt]
  Seoul National University \\
  \makebox[0pt]{\texttt{erickun0125@snu.ac.kr \quad \{bdlim,jblim97\}@robotics.snu.ac.kr \quad \{cms0714,fcp\}@snu.ac.kr}}
}

\begin{document}

\maketitle

\begin{abstract}
    Articulated object manipulation tasks such as opening doors, tightening bolts, and pulling drawers impose kinematic constraints on a robot end-effector's motion that neural policies may fail to satisfy.
Unlike hybrid force-position control that strictly separates force and position control into orthogonal subspaces, impedance control modulates a mechanical impedance system in the form of a virtual inertia, damping, and stiffness.
However, selecting task-appropriate impedance gains remains challenging for robot policies.
In this paper, we address this problem for the case of manipulating articulated objects.
Our constraint-decomposed impedance control method uses the sampling diversity of a generative policy to estimate the feasible motion directions in twist space.
The key idea is to infer the feasible and constrained subspaces of twists from an ensemble of trajectories generated by the policy from the same observation.
Principal component analysis (PCA) of the end-effector twists determines the feasible twist subspace, which is then used to form $6 \times 6$ impedance gain matrices.
For many common articulated objects, the feasible twist subspace is independent of the object configuration, so the subspace and resulting gains are estimated once and held fixed during manipulation, minimizing computational overhead.
For less common cases where this independence does not hold, we instead estimate a configuration-independent twist subspace of higher dimension that contains the local feasible subspaces, allowing the same fixed-gain strategy to be used.
Experiments across simulated and real-world single- and multi-DoF objects show that the proposed method consistently reduces constraint-violating contact loads relative to isotropic impedance control while maintaining task execution.
Project page:\newline\url{https://erickun0125.github.io/constraint-decomposed-impedance/}

\end{abstract}


\section{Introduction}
\label{sec:intro}

\begin{figure}[t]
    \centering
    \includegraphics[width=\linewidth]{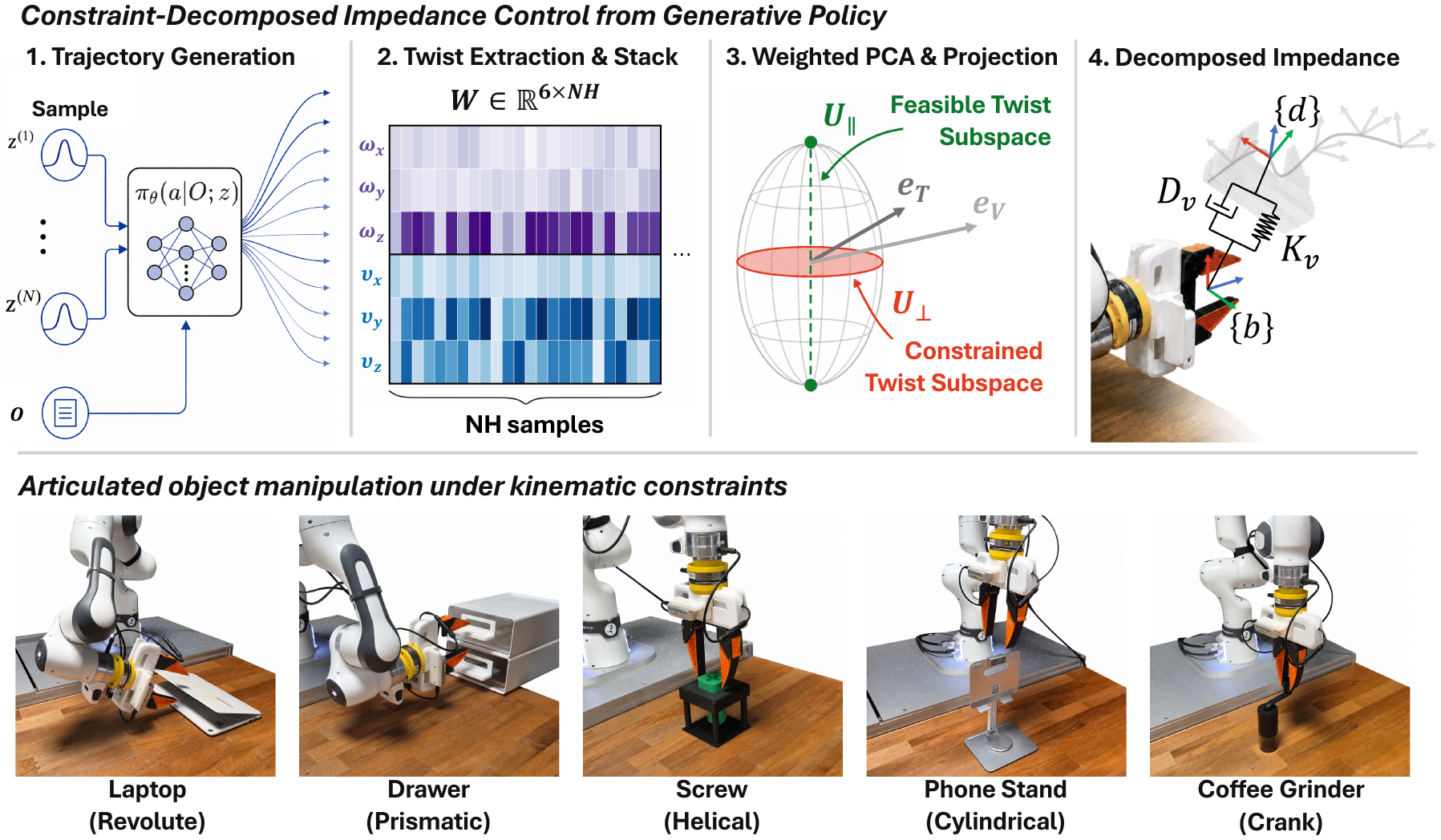}
    \caption{\textbf{Overview of the proposed constraint-decomposed impedance control.}
    Top: (1) At grasp closure, the generative policy $\pi_\theta(a| O;z)$ samples multiple trajectories from the same observation $O$.
    (2) The trajectories are converted to end-effector body twists and stacked into a data matrix $W\in\mathbb{R}^{6\times NH}$.
    (3) A weighted PCA recovers the feasible twist subspace $U_\parallel$ and its complementary constrained subspace $U_\perp$, (4) which define the constraint-decomposed impedance gains $K_v, D_v$ that are stiff along $U_\parallel$ and compliant along $U_\perp$.
    Bottom: the articulated objects used in our real-world experiments, spanning revolute, prismatic, helical, cylindrical, and crank mechanisms.}
    \label{fig:overview}
    \vspace{-0.25em}
\end{figure}

Opening doors, tightening bolts, and pulling drawers all require a robot end-effector to move according to the kinematic constraints of the object.
A learned policy, however, is not guaranteed to output a trajectory that exactly satisfies these constraints~\cite{xu2024leto, foland2025howwell}.
When tracked with high stiffness, even minor constraint-violating motions can produce large contact and internal forces.
This makes compliant execution important when applying learned policies to articulated objects.

Impedance control~\cite{hogan1985impedance} provides a natural way to add this compliance.
By lowering the impedance in constrained directions, the robot can avoid forcing motions that the object does not allow.
However, making the robot compliant in all directions is also undesirable, since the robot still needs to transmit enough force along the directions that drive the articulation.
The central issue is therefore not only choosing the magnitude of the impedance gains, but also deciding which motion components should be tracked with high impedance and which should remain compliant.
The core requirement is to identify the constraint geometry of the articulated object, namely which motion directions are feasible and which are constrained, and translate this geometry into \textit{constraint-decomposed} impedance gains.
However, using such constraint geometry to design direction-wise impedance gains remains underexplored in policy-driven articulated object manipulation.

For articulated objects, the constraints have a simple structure in twist space: the feasible body twists at each object configuration form a linear subspace of $\se(3)$~\cite{lynch2017modern}.
Furthermore, for many common mechanisms, including revolute, prismatic, helical, cylindrical, and planar joints, the subspace of locally feasible body twists does not change with object configuration.
For less common cases where the feasible twist subspace varies with object configuration, there still exists a higher-dimensional, configuration-independent subspace that envelops all the locally feasible twists.
Once identified, such a subspace can be reused during manipulation to define how impedance gains should be decomposed.

One way to identify the structure of feasible motions is to explicitly estimate the screw axes of the object from vision, force/torque sensing, physical interaction, or demonstrations~\cite{sturm2011probabilistic, katz2013interactive, hausman2015active, liu2019learning, martinmartin2019coupled, jain2021screwnet, jain2022distributional, mahalingam2023human, kim2025screwsplat, buchanan2026online}.
These works have shown promising results, but require separate perception or estimation modules.
Moreover, the estimated screw axes must be aligned with respect to a reference frame before they can be used for manipulation.
Another approach is to estimate the constraints in terms of robot motion from demonstrations, but the results are not reusable when the robot grasps the object at a different location.

To address these limitations, the first part of our method estimates the feasible twist subspace directly from output trajectories of a generative policy, without requiring explicit screw-axis estimation or a demonstration aligned with the current grasp pose.
For our purposes, a single output trajectory does not provide sufficient information, because it may contain constraint-violating components and may not cover all feasible directions, especially for multi-DoF objects.
We therefore use multiple trajectories sampled from a generative policy.
From the same observation, the policy generates multiple candidate trajectories, and we convert these trajectories into end-effector body twists and use their dominant subspace as an estimate of the feasible twist subspace.

In the second part of our method, the estimated subspace is used to construct the constraint-decomposed impedance gains (Figure~\ref{fig:overview}).
High impedance gains are assigned along the feasible subspace so that the robot stiffly follows the motions that drive the object, whereas low gains are assigned to the complementary subspace so that constraint-violating motions are not enforced.
The subspace and the resulting gains are estimated once at the beginning of manipulation and held fixed afterward, minimizing computational overhead.

We evaluate our approach in simulated and real-world environments across single- and multi-DoF articulated objects, covering objects with configuration-independent and configuration-dependent feasible twist subspaces.
The experimental results demonstrate that the use of constraint-decomposed gains substantially reduces contact forces along the constraint-violating directions when compared to a baseline controller with isotropic gains while maintaining or improving task execution.

\section{Related Work}
\label{sec:related_work}

\paragraph{Explicit articulation estimation.}
A common way to obtain constraint information is to explicitly estimate the object's articulation structure.
Vision-based methods estimate joint axes, screw parameters, or articulation models from depth or RGB observations~\cite{jain2021screwnet, jain2022distributional, kim2025screwsplat}.
Interactive perception methods estimate kinematic models by moving the object and observing its motion, often using visual, proprioceptive, or force/torque feedback~\cite{sturm2011probabilistic, katz2013interactive, hausman2015active, martinmartin2019coupled, buchanan2026online}.
Demonstration-based methods infer articulated constraints or screw-motion structure from human demonstrations~\cite{liu2019learning, mahalingam2023human}.
These estimators provide useful geometric information, but the estimated model is obtained in a separate stage and must be aligned with the grasp and controller frames before it can be used for impedance control.
Instead, our method requires no separate articulation perception module or physical probing.

\paragraph{Variance-based impedance and direction-wise control.}
Calinon et al.~\cite{calinon2010learning} mapped positional demonstration variance to stiffness, assigning high stiffness to low-variance directions.
In contrast, we analyze variation in body-twist space, where low variation along a twist direction means that the object rarely moves in that direction and therefore calls for low impedance.
SwRL~\cite{kim2024swrl} uses a task-space decomposition for hybrid motion-force control in articulated object manipulation, but the decomposition is given as prior knowledge rather than estimated from the object motion.
ACP~\cite{hou2025adaptive} sets low stiffness along the external-force direction for contact-rich manipulation, but this direction is defined by a force-based rule.
In contrast, our method sets direction-wise impedance according to the constraint geometry of the articulated object, rather than from a predefined rule.

\section{Preliminaries}
\label{sec:background}

\subsection{Operational-Space Impedance Control}
\label{sec:bg_impedance}

Operational-space impedance control makes the robot end-effector respond compliantly to external forces while tracking a desired task-space trajectory.
To this end, it specifies a virtual mass-spring-damper relation between the task-space tracking error and the external wrench as
\begin{equation}
    M_v \dot e_V + D_v e_V + K_v e_T = F_{\rm ext},
    \label{eq:bg_target_impedance}
\end{equation}
where $e_T,e_V\in\mathbb{R}^6$ are the task-space pose and twist errors, $F_{\rm ext}$ is the external wrench, and $M_v,D_v,K_v\in\mathbb{R}^{6\times 6}$ are virtual inertia, damping, and stiffness.
$M_v$ is usually set to match the task-space inertia $\Lambda(q)=(J(q)M(q)^{-1}J(q)^\top)^{-1}$ to compensate the effect of robot dynamics, where $J$ is the robot Jacobian and $M$ is the joint-space inertia.
$D_v$ and $K_v$ are user-defined damping and stiffness parameters for handling trade-off between accurate tracking and force compliance.
The target impedance behavior of the end-effector is realized by a task-space feedback wrench $F_{\mathrm{fb}}$.
To achieve this, the total commanded task-space wrench $F_{\mathrm{cmd}}$ is formed by combining $F_{\mathrm{fb}}$ with a dynamic feedforward term $F_{\mathrm{ff}}$ compensating for the robot dynamics:
\begin{equation}
    F_{\mathrm{cmd}} = F_{\mathrm{ff}} + F_{\mathrm{fb}}, \quad F_{\mathrm{fb}} = K_v e_T + D_v e_V,
    \label{eq:bg_commanded_wrench}
\end{equation}
where $F_{\mathrm{cmd}}$ is then mapped to joint torques as $\tau = J^\top F_{\mathrm{cmd}}$.

Intuitively, high gains promote accurate trajectory tracking for task execution, whereas low gains allow compliance to avoid excessive contact and internal forces.
In impedance control for articulated-object manipulation, the key is therefore to distinguish the feasible directions that should receive high gains from the constrained directions that should remain compliant with low gains.
Finding these directions is the main problem addressed in this paper.

\subsection{Kinematics of Articulated Objects}
\label{sec:bg_kinematics}

We model an articulated object as a $k$-DoF open chain of rigid links, with the base link fixed to the environment and the distal link rigidly grasped by an end-effector.
We attach a space frame $\{s\}$ to the base, and a body frame $\{b\}$ to the end-effector.
We denote the object configuration by $q\in\mathbb{R}^k$, which determines $T_{sb}(q)\in\SE(3)$, the pose of $\{b\}$ expressed in $\{s\}$.
Following open-chain kinematics~\cite{lynch2017modern}, any constraint-satisfying body twist $V_b = J_b(q)\dot q$ must lie within the \textit{local feasible subspace}
\begin{equation}
    \mathcal{S}(q) := \operatorname{span}\{B_1(q), \ldots, B_k(q)\} \subset \mathbb{R}^6,
    \label{eq:body_twist_kinematics}
\end{equation}
where $J_b(q)\in\mathbb{R}^{6\times k}$ is the body Jacobian with columns $B_1(q),\ldots,B_k(q)$.

\begin{wrapfigure}{r}{0.46\textwidth}
    \centering
    \vspace{-10pt}
    \includegraphics[width=0.46\textwidth]{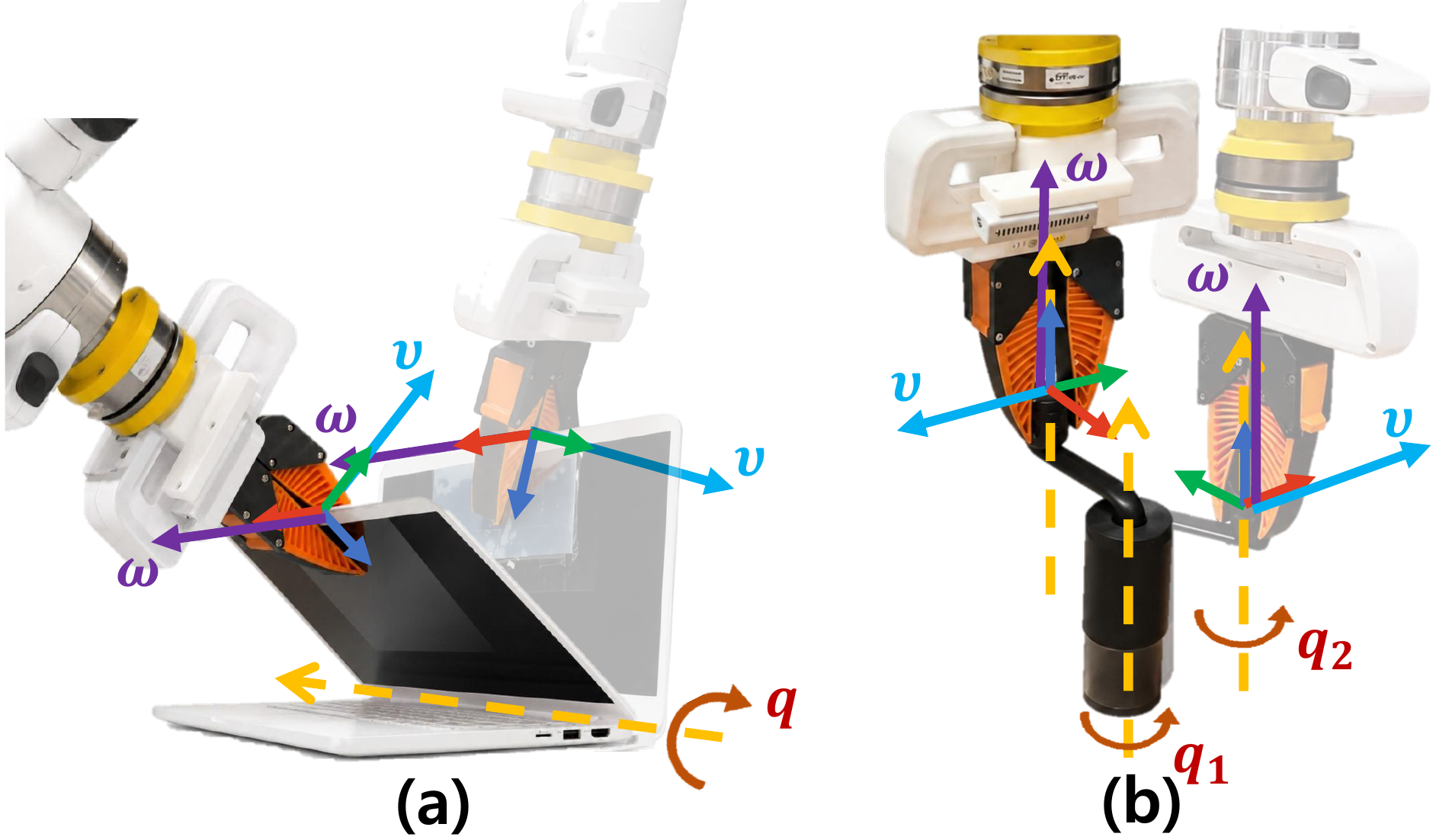}
    \vspace{-8pt}
    \caption{\textbf{Articulated object examples.}
    Visualization of the body frame and feasible body twist $V_b = (\omega, v)$, with joint axes (yellow dashed).
    (a) For a laptop with a revolute joint, $\omega$ always aligns with $+x$ and $v$ with $+y$ (configuration-independent).
    (b) For a coffee-grinder crank, $v$ aligns with $-y$ in solid case, but with $+x$ in translucent case (configuration-dependent).}
    \vspace{-12pt}
    \label{fig:articulated_examples}
\end{wrapfigure}

We distinguish two classes of articulated objects according to whether the local feasible subspace $\mathcal{S}(q)$ depends on the object configuration $q$.
First, in the \emph{configuration-independent} case, $\mathcal{S}(q)$ remains the same for all $q\in\mathbb{R}^k$.
Most single-joint objects, including objects with revolute, prismatic, cylindrical, or spherical joints, fall into this class.
For example, the laptop in Figure~\ref{fig:articulated_examples}(a) is a configuration-independent object because $B_1(q)=(1,0,0,0,l,0)^\top$ for all $q\in\mathbb{R}$ with the lever arm length $l$.
A mathematical sufficient condition for configuration independence is provided in Appendix~\ref{app:lie_subalgebra}.
Second, in the \emph{configuration-dependent} case, $\mathcal{S}(q)$ varies with $q$.
For example, the coffee grinder in Figure~\ref{fig:articulated_examples}(b) is a configuration-dependent object because it has  $B_1(q)=(0,0,1,-l \sin q_2,-l\cos q_2,0)^\top$ and $B_2(q)=(0,0,1,0,0,0)^\top$ with the lever arm length $l$.
Most multi-joint objects fall into this class.

To cover both classes, we define the \textit{global envelope subspace} as
\begin{equation}
    \mathcal{S}_{\rm env}:=\operatorname{span}\!\left(\bigcup_{q\in\mathbb{R}^k}\mathcal{S}(q)\right)\subset\mathbb{R}^6,
    \label{eq:subspace_envelope}
\end{equation}
which is the smallest linear subspace that contains all local feasible subspaces.
Since we do \textit{not} assume prior knowledge of whether a given object is configuration-independent, $\mathcal{S}_{\rm env}$ serves as a natural target for our method to estimate.
For configuration-independent objects, at any configuration $q$, $V_b$ is feasible if and only if $V_b \in \mathcal{S}_{\rm env}$.
However, for configuration-dependent objects, $V_b$ being feasible implies $V_b \in \mathcal{S}_{\rm env}$ but the converse does not hold.

\section{Method}
\label{sec:method}

\begin{algorithm}[t]
    \caption{Constraint-Decomposed Impedance Gains from Generative Policy}
    \label{alg:constraint_decomposed}
    \begin{algorithmic}[1]
        \setlength{\itemsep}{2pt}

        \Statex \hspace*{-\algorithmicindent} \textbf{Given:} policy $\pi(a \mid O)$; inertia metric $\bar\Lambda$; timestep $\Delta t$
        \Statex \hspace*{-\algorithmicindent} \textbf{Hyperparameters:} ensemble size $N$; horizon $H$; gains $k_\parallel, k_\perp, d_\parallel, d_\perp$; scale bounds $\alpha_{\min}, \alpha_{\max}$; thresholds $v_{\rm abs}, \kappa_{\rm rel}$
        \Statex \hspace*{-\algorithmicindent} \textbf{Input:} observation $O$ \hfill \textbf{Output:} gains $K_v, D_v$

        \For{$i = 1, \ldots, N$}
            \State $a^{(i)} = (T^{(i)}_1, \ldots, T^{(i)}_H) \sim \pi(\cdot \mid O)$ \Comment{sample trajectory ensemble}
            \State $V^{(i)}_{b,t} \gets \frac{1}{2 \Delta t} \big[ \log \big( (T^{(i)}_{t-1})^{-1} T^{(i)}_{t+1} \big) \big]^\vee$, \ $t = 1, \ldots, H$ \Comment{one-sided at $t{=}1,H$}
        \EndFor
        \State $W \gets [\, V^{(1)}_{b, 1} \, \cdots \, V^{(N)}_{b, H} \,]$
        \State $u = [\, u_\omega^\top \, u_v^\top \,]^\top \gets$ first left singular vector of $W$;\ \ $\alpha \gets \operatorname{clip}\!\big( \|u_v\| / \|u_\omega\|, \alpha_{\min}, \alpha_{\max} \big)$
        \State $L_\alpha \gets \operatorname{blkdiag}(\alpha I_3, I_3)$;\ \ $\tilde{U} \tilde{\Sigma} \tilde{V}^{\top} \gets \operatorname{SVD}(L_\alpha W)$;\ \ $U = [\, u_1 \, \cdots \, u_6 \,] \gets L_\alpha^{-1} \tilde{U}$
        \State $\lambda_k \gets \tilde\sigma_k^2 / (NH)$,\ \ $k = 1, \ldots, 6$ \Comment{$\tilde\sigma_k$: $k$-th singular value}
        \State $m \gets \#\{ k : \sqrt{\lambda_k} > \max(\kappa_{\rm rel}\sqrt{\lambda_6},\, v_{\rm abs}) \}$;\ \ $U_\parallel \gets [\, u_1 \, \cdots \, u_m \,]$ \Comment{dimension selection}
        \State $P^{\bar\Lambda}_\parallel \gets U_\parallel ( U_\parallel^\top \bar\Lambda U_\parallel )^{-1} U_\parallel^\top \bar\Lambda$,\ \ $P^{\bar\Lambda}_\perp \gets I_6 - P^{\bar\Lambda}_\parallel$
        \State \Return $K_v \gets \bar\Lambda ( k_\parallel P^{\bar\Lambda}_\parallel + k_\perp P^{\bar\Lambda}_\perp )$,\ \ $D_v \gets \bar\Lambda ( d_\parallel P^{\bar\Lambda}_\parallel + d_\perp P^{\bar\Lambda}_\perp )$
    \end{algorithmic}
\end{algorithm}

We now introduce our constraint-decomposed impedance control framework.
We assume we have a generative policy that can manipulate an articulated object, but we do \textit{not} modify the policy and do \textit{not} assume any prior knowledge of the object's kinematic model, including DoF, joint types, or screw axes.
Our goal is to (i) estimate the feasible body twist directions using the policy outputs and (ii) convert the estimate into decomposed impedance gains for effective and safe manipulation of articulated objects.
The overall pipeline is summarized in Algorithm~\ref{alg:constraint_decomposed}.

\subsection{Constraint Estimation from a Trajectory Ensemble}
\label{sec:constraint_estimation}

We assume access to a generative policy $\pi(a \mid O)$ for manipulating an articulated object, where the action $a=(T_1,\ldots,T_H)$, $T_t\in{\rm SE}(3)$ is an $H$-step desired pose trajectory for the end-effector frame $\{b\}$ expressed in the space frame $\{s\}$, and the observation $O$ consists of robot proprioception and visual inputs.
The goal in this subsection is to estimate the global envelope subspace ${\cal S}_{\rm env}$ from the trajectories generated by the policy.

\paragraph{Body twist extraction.}
At a given observation $O$, we sample a set of $N$ action trajectories $a^{(i)}=(T^{(i)}_1,\ldots,T^{(i)}_H)$ by the policy, which we refer to as a \textit{trajectory ensemble}.
We then compute body twists $V^{(i)}_{b,t}$ at each $T_t^{(i)}$ by taking finite differences on the trajectories, and collect them into a twist matrix $W=[\,V^{(1)}_{b,1}\ \cdots\ V^{(N)}_{b,H}\,]\in\mathbb R^{6\times NH}$.
We consider a trajectory ensemble rather than a single trajectory for two reasons.
First, a single trajectory may contain constraint-violating components, while a trajectory ensemble is more likely to reveal the constraint-consistent directions across samples.
Second, for multi-DoF objects, a single trajectory may not span all feasible directions on its own.

\paragraph{Weighted PCA in twist space.}
The collected body twists in $W$ are expected to lie close to the global envelope subspace ${\cal S}_{\rm env}$, which motivates an (uncentered) principal component analysis (PCA) to estimate ${\cal S}_{\rm env}$.
In order to perform PCA in ${\rm se}(3)$, the space of twists, we first need to choose an \textit{inner product} on ${\rm se}(3)$~\cite{park1995distance}.
Let $G\in\mathbb{R}^{6\times6}$ be a symmetric positive-definite matrix, and define
\begin{equation}
    \langle V_1,V_2\rangle_G := V_1^\top G V_2,
    \qquad
    \|V\|_G^2 := V^\top G V .
\end{equation}
Given a Cholesky decomposition $G=L^\top L$ with $L\in\mathbb{R}^{6\times6}$, the $G$-weighted PCA can be done by applying singular value decomposition (SVD) to the transformed twists $LW\in\mathbb{R}^{6\times NH}$ as
\begin{equation}
    LW = \widetilde U \widetilde\Sigma \widetilde V^\top,
    \qquad
    U = L^{-1}\widetilde U,
\end{equation}
where $\widetilde U,U\in\mathbb{R}^{6\times6}$, $\widetilde\Sigma\in\mathbb{R}^{6\times NH}$, and $\widetilde V\in\mathbb{R}^{NH\times NH}$.
Let $\widetilde\sigma_1,\ldots,\widetilde\sigma_6$ denote the singular values in $\widetilde\Sigma$.
Ideally, we would have $\widetilde\sigma_1\geq\cdots\geq\widetilde\sigma_m>0$, $\widetilde\sigma_{m+1}=\cdots=\widetilde\sigma_6=0$, where $m=\dim\mathcal{S}_{\rm env}$.
However, in practice, we have $\widetilde\sigma_{m+1},\ldots,\widetilde\sigma_6>0$, because the sampled twists contain off-manifold components outside $\mathcal S_{\rm env}$.
To separate them from the feasible motion, we normalize the squared singular values by the number of twists, $\lambda_k=\widetilde\sigma_k^2/M$ with $M=NH$, so that $\sqrt{\lambda_k}$ is the RMS grasp-point speed along $u_k$, and we threshold this speed in two ways.
An absolute threshold $v_{\rm abs}$ targets sources whose off-manifold speed has a roughly fixed magnitude for a given model and sampler, such as the generation limit of the diffusion policy.
A relative threshold $\kappa_{\rm rel}$, which compares each speed with $\sqrt{\lambda_6}$, targets sources that lift the whole off-manifold part of the spectrum, including $\lambda_6$, such as grasp-chain deviations due to compliance and clearance between gripper and object.
We retain the leading directions that pass both tests, $m=\#\{k:\sqrt{\lambda_k}>\max(\kappa_{\rm rel}\sqrt{\lambda_6},\,v_{\rm abs})\}$.
Then, we use the first $m$ columns of $U$ to construct the estimated global envelope subspace $\widehat{\mathcal S}_{\rm env}$ and its basis $U_\parallel$ as
\begin{equation}
    \widehat{\mathcal S}_{\rm env}:=\operatorname{span}\{u_1,\ldots,u_m\}, \qquad U_\parallel = [\,u_1 \cdots u_m\,].
\end{equation}

\paragraph{Choice of the inner product.}
A natural choice of inner product is $G=\operatorname{blkdiag}(\alpha^2 I_3, I_3)$, where $\alpha>0$ is a characteristic length that puts angular and linear velocities on a common scale~\cite{park1995distance}.
In this case, the $G$-weighted PCA is equivalent to ordinary PCA on the scaled twists $(\alpha\omega,v)$.
In our implementation, we find the first left singular vector $u=[u_\omega^\top,u_v^\top]^\top$ of the unweighted twist matrix $W$, and then set $\alpha = \operatorname{clip}\!\left( \|u_v\|/\|u_\omega\|, \alpha_{\min}, \alpha_{\max} \right)$ with predefined $\alpha_{\min}, \alpha_{\max}$.

\subsection{Impedance Gain Decomposition}
\label{sec:gain_decomposition}

Given the estimated feasible subspace $\widehat{\mathcal S}_{\rm env}$, we now aim to construct impedance gains that are stiff along this subspace and compliant along its complement.
Since our control inputs are wrenches, while the constraints are defined in the twist space, a Euclidean projection without considering the robot's physical inertia would cause dynamic coupling.
We therefore measure orthogonality with a symmetric positive-definite inertia metric  $\bar\Lambda\in\mathbb R^{6\times 6}$, interpreted from a nominal operational-space inertia reduced to a block-isotropic form (Appendix~\ref{app:normalized_block_iso}).
This metric sets the translational-rotational scaling and links the twist split to the wrench split in a dynamically consistent way.
Because the estimated basis $U_\parallel$ spans a kinematic subspace, we weight it by the inertia metric $\bar{\Lambda}$ to map the tracking errors directly \vspace{0.2em} to the wrench level.
We therefore construct the dynamically consistent feedback projectors onto $\widehat{\mathcal{S}}_{\rm env}$:
\begin{equation}
    P_\parallel^{\bar\Lambda}
    =
    U_\parallel
    \left(
    U_\parallel^\top \bar\Lambda U_\parallel
    \right)^{-1}
    U_\parallel^\top \bar\Lambda,
    \qquad
    P_\perp^{\bar\Lambda}
    =
    I-P_\parallel^{\bar\Lambda} .
    \label{eq:lambda_projectors}
\end{equation}
Because the decomposed gains act on the feasible subspace in twist space, the tracking errors must share that geometry.
We therefore use body-frame group errors~\cite{BulloMurray1995, BulloMurray1999}.
These consist of a double-geodesic pose error $e_T$ and an adjoint-transported twist error $e_V$,
\begin{equation}
    e_T =
    \begin{bmatrix}
    \log(R_b^\top R_d)^\vee \\[2pt]
    R_b^\top(p_d-p_b)
    \end{bmatrix},
    \qquad
    e_V = \Ad_{T_b^{-1}T_d}\, V_d - V_b ,
    \label{eq:tracking_errors}
\end{equation}
with current pose $T_b=(R_b,p_b)$, desired pose $T_d=(R_d,p_d)$, and desired body twist $V_d$, all in the same body-frame coordinates as the projectors and gains.
For small errors, $e_T$ and $e_V$ reduce to the ordinary pose and twist errors of operational-space impedance control (Appendix~\ref{app:tracking_errors}).

The decomposed impedance gains for the control law \eqref{eq:bg_commanded_wrench} are
\begin{equation}
    K_v
    =
    \bar\Lambda
    \left(
    k_\parallel P_\parallel^{\bar\Lambda}
    +
    k_\perp P_\perp^{\bar\Lambda}
    \right),
    \qquad
    D_v
    =
    \bar\Lambda
    \left(
    d_\parallel P_\parallel^{\bar\Lambda}
    +
    d_\perp P_\perp^{\bar\Lambda}
    \right).
    \label{eq:decomposed_gains}
\end{equation}
With $k_\parallel>k_\perp$ and $d_\parallel>d_\perp$, the feedback is stiff along $\widehat{\mathcal S}_{\rm env}$ and compliant along its complement.

Because $P_\parallel^{\bar\Lambda}$ and $P_\perp^{\bar\Lambda}$ are $\bar\Lambda$-orthogonal, $K_v$ and $D_v$ are symmetric positive semidefinite for \vspace{-1.5pt} nonnegative gains; with $\bar\Lambda$ and $\widehat{\mathcal S}_{\rm env}$ held fixed during the articulation phase, they derive from a quadratic potential and dissipation that make the feedback passive in the small-error regime (Appendix~\ref{app:e_passivity}).

\section{Experiments}
\label{sec:experiments}

In simulated and real-world environments, we demonstrate that (i) trajectory ensembles provide accurate estimates of feasible twist subspaces, and (ii) constraint-decomposed impedance control reduces the contact force and moment associated with constrained directions while maintaining task execution.
Experimental details are provided in Appendix~\ref{app:exp_config}.

\paragraph{Evaluation metrics.}
We measure contact force and moment, and project each part onto the linear and angular complements induced by the reference feasible twist subspace.
Since the grasp pose varies across rollouts, the reference subspace is the analytic feasible subspace at each rollout's grasp pose, in simulation and in the real world, and all controllers are scored over a common articulation range.
\emph{Impulse} (ICF, ICM) and \emph{Peak} (PCF, PCM) \emph{Constraint-violating Force/Moment} report, respectively, the time-integral (N$\cdot$s, N$\cdot$m$\cdot$s) and the $95$th-percentile magnitude (N, N$\cdot$m) of these complementary force and moment components.
For configuration-dependent objects such as the universal joint and crank, the reference is the global envelope $\mathcal{S}_{\rm env}$, so these metrics measure load outside the envelope rather than local instantaneous infeasibility.
\emph{Subspace Error} (SE), reported for \textbf{Ours}, is the full-subspace maximum principal angle between $\widehat{\mathcal S}_{\rm env}$ and the reference subspace.
\emph{Success Rate} (SR), reported in simulation, is the fraction of trials in which the object completes this common articulation range.

\paragraph{Compared controllers.}
\textbf{Iso} is operational-space impedance control with isotropic gains, and \textbf{Ours} is the proposed method, which keeps the same impedance as \textbf{Iso} along the estimated feasible subspace (both rotational and translational components) and lowers it only along the constrained complement (Appendix~\ref{app:iso_ours_oracle_gains}).
In simulation, we additionally include \textbf{Oracle}, which uses our gain decomposition with the grasp-conditioned analytic reference subspace.
For fair comparison, all controllers run on the same Diffusion Policy~\cite{chi2025diffusion} checkpoint, trained from scratch for each object.

\begin{wrapfigure}{r}{0.4\textwidth}
    \centering
    \vspace{-35pt}
    \includegraphics[width=0.4\textwidth]{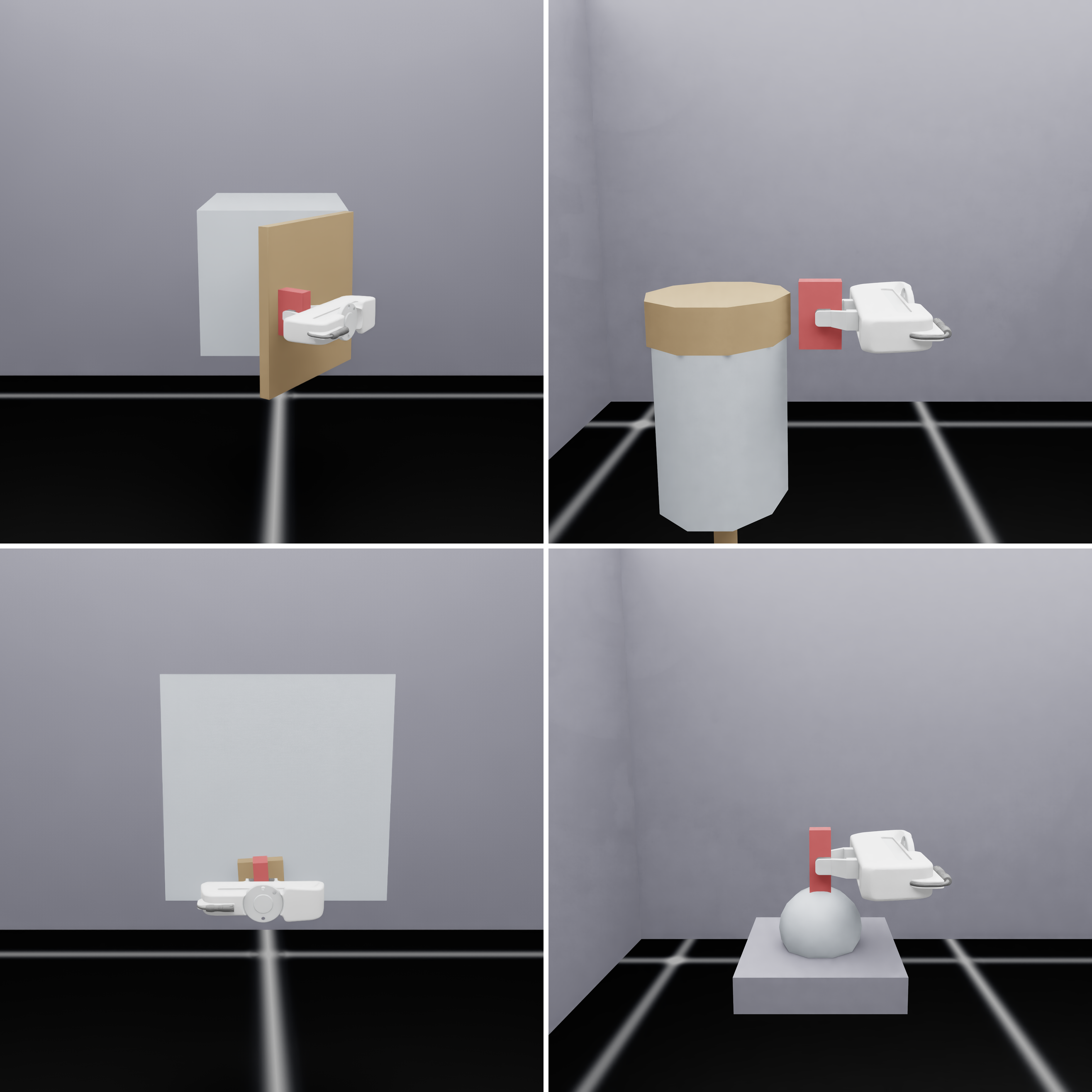}
    \caption{\textbf{Simulated articulated objects.}
    Revolute, cylindrical, planar, and universal joints in Isaac Sim, manipulated by a flying gripper.}
    \label{fig:sim_objects}
    \vspace{-25pt}
\end{wrapfigure}

\subsection{Simulation Experiments}
\label{sec:exp_sim}

\paragraph{Setup.}
We set up four articulated objects (Figure~\ref{fig:sim_objects}) in Isaac Sim~\cite{NVIDIA_Isaac_Sim} with revolute, cylindrical, planar, and universal joints.
The first three are configuration-independent, with local feasible dimension $1$, $2$, and $3$; the universal joint is configuration-dependent, with local feasible dimension $2$ and global envelope dimension $3$.
For all objects, a flying gripper grasps the object and drives it toward a target configuration; the policy outputs pose trajectories of $H=24$ steps at $0.1$~s intervals, and the impedance controller runs at $150$~Hz.
At grasp closure, \textbf{Ours} samples $N=128$ trajectories with $10$ DDIM denoising steps and selects the subspace dimension with $v_{\rm abs}=5$~mm/s and $\kappa_{\rm rel}=4$; \textbf{Ours} and \textbf{Oracle} scale the stiffness along the constrained complement by $r=0.1$ (Appendix~\ref{app:iso_ours_oracle_gains}).
For each object we run every controller from $100$ distinct initial configurations, giving $100$ attempts per controller.

\paragraph{Results.}
Table~\ref{tab:sim_main} shows two consistent trends.
First, \textbf{Ours} reduces the constraint-violating force on every task: the force impulse falls to $0.13$--$0.37{\times}$ of Iso and the peak force to $0.18$--$0.46{\times}$.
Second, Ours recovers at least $90\%$ of the Oracle's reduction relative to Iso in both force metrics, showing that the trajectory ensemble identifies the feasible twist directions accurately enough for control; this is corroborated by the small SE ($1.64^\circ$--$4.36^\circ$, largest on the configuration-dependent universal joint).
Moment loads are small overall, as they are taken about the grasp point; \textbf{Ours} still lowers the moment impulse, and the universal joint has a structurally empty angular complement because its envelope contains all three rotational DoFs.
These trends carry over to a single multi-task policy trained on all four objects, with which \textbf{Ours} reduces the force impulse to $0.14$--$0.37{\times}$ and the peak force to $0.20$--$0.49{\times}$ of Iso (Appendix~\ref{app:sim_multitask}).

\begin{table}[t]
\centering
\caption{ \textbf{Simulation results.}
Entries are Iso / Ours / Oracle; SE reports Ours only, and boldface marks the better of Iso and Ours.
Dashes indicate structurally empty complement components.
}
\label{tab:sim_main}
\vspace{0.25em}
\scriptsize
\setlength{\tabcolsep}{2.5pt}
\resizebox{\linewidth}{!}{%
\begin{tabular}{l c c c c c c}
\toprule
Task & SR (\%) & SE (\textdegree) & ICF (N$\cdot$s) & PCF (N) & ICM (N$\cdot$m$\cdot$s) & PCM (N$\cdot$m)\\
\midrule
Revolute & \phantom{0}99 / \phantom{0}99 / \phantom{0}99 & 1.64 & 7.18 / \textbf{2.62} / 2.13 & 3.38 / \textbf{1.56} / 1.48 & .29 / \textbf{.22} / .22 & .19 / \textbf{.18} / .17\\
Cylindrical & \phantom{0}99 / \textbf{100} / 100 & 1.73 & 5.31 / \textbf{0.83} / 0.53 & 1.96 / \textbf{0.36} / 0.31 & .32 / \textbf{.22} / .21 & .16 / \textbf{.14} / .14\\
Planar & \phantom{0}99 / \phantom{0}99 / \phantom{0}99 & 2.78 & 15.3 / \textbf{1.93} / 2.20 & 4.03 / \textbf{1.18} / 1.16 & .74 / \textbf{.48} / .34 & .29 / \textbf{.23} / .21\\
Universal & 100 / 100 / 100 & 4.36 & 2.18 / \textbf{0.46} / 0.34 & 1.25 / \textbf{0.23} / 0.19 & -- / -- / -- & -- / -- / --\\
\bottomrule
\end{tabular}
}
\vspace{-0.75em}
\end{table}

\subsection{Real-World Experiments}
\label{sec:exp_real}

\paragraph{Setup.}
We set up four single-joint objects; revolute, prismatic, and helical objects with one-dimensional feasible subspaces, and a cylindrical object with a two-dimensional feasible subspace.
We also include a coffee grinder with two parallel revolute joints (\textit{crank}), a configuration-dependent case with local feasible dimension $2$ and global envelope dimension $3$.
For all objects (Figure~\ref{fig:overview}, bottom), a Franka Research 3 (FR3) grasps the object and drives it toward a target configuration, with the same trajectory horizon and step interval as in simulation and the impedance controller at $1$~kHz.
At grasp closure, \textbf{Ours} samples $N=32$ trajectories with $100$ DDPM denoising steps and selects the subspace dimension with $v_{\rm abs}=3$~mm/s and $\kappa_{\rm rel}=10$ ($\kappa_{\rm rel}=4$ for the crank; Appendix~\ref{app:hyperparams}), using the same $r=0.1$.
We evaluate only \textbf{Iso} and \textbf{Ours}.
For each object, we fix the initial object configuration and base pose, and run each controller for $10$ trials, yielding $50$ trials per controller.
For a fair wrench comparison at matched task progress, we tune the impedance gains and the FR3 collision-detection thresholds so that both \textbf{Iso} and \textbf{Ours} succeed on all trials, and integrate the metrics from grasp completion to a common articulation target.

\paragraph{Results.}
The real-world trend matches simulation (Table~\ref{tab:real_main}).
Averaged over the objects, \textbf{Ours} lowers the constraint-violating force impulse and peak force to $0.72{\times}$ and $0.70{\times}$ of \textbf{Iso}, and the moment impulse and peak moment to $0.69{\times}$ and $0.56{\times}$ (prismatic excluded for the moments).
On the four configuration-independent objects, \textbf{Ours} improves every metric, with the force impulse at $0.53$--$0.93{\times}$ of \textbf{Iso} and the peak force at $0.42$--$0.96{\times}$.
The SE of \textbf{Ours} is $2.89^\circ$--$5.76^\circ$ on four objects and $13.01^\circ$ on revolute.
On the crank, the improvement is smaller: \textbf{Ours} still lowers the peak force and both moment metrics (to $0.74$--$0.93{\times}$), but its force impulse rises to $1.14{\times}$ that of \textbf{Iso}, because the fixed envelope aggregates directions feasible across the articulation, exceeding the instantaneous feasible subspace, and can assign high impedance to directions constrained at the current configuration.

\begin{table}[t]
\centering
\caption{ \textbf{Real-world results.}
Entries are Iso / Ours (Ours/Iso); SE reports Ours only, and boldface marks the better of Iso and Ours.
Daggers mark full moments, as prismatic has no feasible rotation.
}
\label{tab:real_main}
\vspace{0.25em}
\scriptsize
\setlength{\tabcolsep}{5.5pt}
\resizebox{\linewidth}{!}{%
\begin{tabular}{l c c c c c}
\toprule
Task & SE (\textdegree) & ICF (N$\cdot$s) & PCF (N) & ICM (N$\cdot$m$\cdot$s) & PCM (N$\cdot$m)\\
\midrule
Revolute & 13.01 & 46.6 / \textbf{25.2} (0.54) & 8.54 / \textbf{3.62} (0.42) & 5.00 / \textbf{3.34} (0.67) & 0.93 / \textbf{0.37} (0.40) \\
Prismatic & \phantom{0}4.74 & 12.3 / \textbf{8.79} (0.72) & 4.63 / \textbf{3.51} (0.76) & 1.84 / \textbf{1.46} (0.80)$^\dagger$ & 0.61 / \textbf{0.55} (0.90)$^\dagger$ \\
Helical & \phantom{0}2.89 & 6.96 / \textbf{6.50} (0.93) & 4.56 / \textbf{4.40} (0.96) & 1.82 / \textbf{0.60} (0.33) & 1.08 / \textbf{0.45} (0.41) \\
Cylindrical & \phantom{0}5.76 & 8.35 / \textbf{4.42} (0.53) & 3.37 / \textbf{2.20} (0.65) & 1.09 / \textbf{0.82} (0.75) & 0.45 / \textbf{0.37} (0.81) \\
Crank & \phantom{0}5.70 & \textbf{20.8} / 23.6 (1.14) & 4.62 / \textbf{4.28} (0.93) & 4.99 / \textbf{4.15} (0.83) & 1.12 / \textbf{0.82} (0.74) \\
\bottomrule
\end{tabular}
}
\vspace{-1.00em}
\end{table}

\subsection{Analysis of Constraint Estimation}
\label{sec:exp_analysis}

We analyze the constraint estimation of Section~\ref{sec:constraint_estimation}, using data from the four simulated objects of Section~\ref{sec:exp_sim}: how its accuracy depends on the ensemble size, whether the diversity of the samples matters, and how robust it is to noise.

\paragraph{Ensemble size and runtime.}
We re-estimate from $N$-sample subsets of each $128$-sample ensemble (Appendix~\ref{app:sim_ensemble_size}).
With the subspace dimension fixed to its true value, the 95th-percentile SE pooled over the objects decreases as $N$ grows: it is $78.9/23.1/8.59/7.23^\circ$ at $N{=}1/2/4/8$ and $6.61^\circ$ at $N{=}128$.
Dimension selection also needs a large enough ensemble: a single threshold setting selects the true dimension on every trial only for $N\ge64$.
Sampling $N{=}128$ trajectories in one batched pass with $10$ DDIM steps and building the gains takes a median of $70.1$~ms and at most $76.5$~ms on an RTX~5070~Ti; this cost is incurred once per trial, at grasp closure, and fits within one action period of the $10$~Hz policy.

\paragraph{Sample diversity.}
We next ask whether the diversity of the sampled trajectories, and not only their number, determines the estimation accuracy.
We measure the diversity of a set of trajectories by their mean pairwise dynamic time warping (DTW) distance, and re-estimate from the $k$ samples of each $128$-sample ensemble that maximize or minimize it (Appendix~\ref{app:sim_diversity}).
Averaged over the objects (mean of per-object median SE), the full ensemble gives an SE of $2.34^\circ$; at $k{=}2/4/8$, the most diverse subsets give $5.91/3.32/2.74^\circ$, with SE above $10^\circ$ on $16/9/3\%$ of grasps, against $19.7/9.9/4.4^\circ$ and $41/28/14\%$ for the least diverse ones.
At equal size, diverse samples thus give a much more accurate estimate: pose-level diversity of the sampled trajectories provides the twist-space coverage that the estimation needs.
The gap narrows as $k$ grows ($2.61/2.48/2.49^\circ$ against $3.43/2.58/2.43^\circ$ at $k{=}16/32/64$), so increasing $N$ remains a simple safeguard against finite-sample under-coverage.

\paragraph{Robustness to noise.}
We test the estimation against noise in the training data, in the observations, and in the policy inference, and report the pooled median SE at $N{=}128$ (Appendix~\ref{app:sim_noise}).
(i)~\emph{Training data.}
We add noise to the actions of the scripted demonstrations in position ($5$~mm) or orientation ($5^\circ$) and train a state-only policy on them.
SE is then $3.1^\circ$ and $10.6^\circ$, respectively, against $12.0^\circ$ and $16.0^\circ$ when the same noise is added directly to the trajectories sampled from the original policy: the policy trained on noisy demonstrations attenuates the noise instead of passing it on to the estimate.
(ii)~\emph{Observations.}
We perturb the image observation of the original policy at inference with Gaussian noise ($\sigma{=}32/255$), a Gaussian blur ($\sigma{=}2$~px), or a random shift ($\sigma{=}4$~px), and estimate from the resulting ensemble.
SE stays within $0.11^\circ$ of the clean $2.10^\circ$.
(iii)~\emph{Inference.}
We reduce the number of DDIM denoising steps of the original policy, which lowers the accuracy of the sampled trajectories.
From $100$ to $4$ steps, SE changes only from $2.16^\circ$ to $2.26^\circ$.

\section{Conclusion}
\label{sec:conclusion}

In this paper, we introduce a constraint-decomposed impedance controller for policy-driven articulated object manipulation.
The key idea is to use the trajectory ensemble of a generative policy to estimate the feasible motion structure of the object in body twist space, and then use this structure to construct impedance gains that are stiff along feasible directions and compliant along constrained directions.
Our method does not require a separate perception module, physical probing, or prior knowledge of the object kinematic model, and the impedance decomposition is obtained from the motions generated by the policy, rather than specified in advance.
In simulation, the estimated subspaces closely match the analytic reference subspaces, and the resulting gains reduce constraint-violating forces compared with isotropic impedance control while maintaining task success.
In real-world experiments, the same trend appears with clear reductions in both force and moment loads.

\paragraph{Limitations and future work.}
The method relies on trajectory samples that sufficiently cover the feasible motion directions of the object.
This may fail when the demonstration data itself does not contain enough motion diversity, for example when only one way of moving a multi-DoF object is shown.
Even when the demonstrations are diverse, the learned policy may still collapse to a small set of modes and fail to reproduce the full feasible structure.
Our analysis shows that larger and more diverse ensembles lower the subspace error, but a feasible direction that the demonstrations barely exercise stays invisible at any ensemble size; contact wrenches could reveal it, since an ideal constraint reaction is reciprocal to every feasible twist.
For configuration-dependent objects, the fixed envelope subspace still provides a practical approximation, but it is less precise than the instantaneous feasible subspace, which explains the smaller improvement on the crank.
Combining the envelope with locally re-sampled short-horizon ensembles to track the moving feasible subspace is a natural next step.
Our gain construction should be viewed as fixed-projector impedance shaping under a nominal inertia metric, rather than full inertial decoupling of the robot-object system; extending it to smooth online projector updates, discontinuous contact transitions, and tactile/force-based refinement during execution while preserving passivity remains future work.

\acknowledgments{This work was supported in part by IITP-MSIT under Grant \mbox{2022-220480} and Grant \mbox{RS-2022-II220480} (Training and Inference Methods for Goal Oriented AI Agents), in part by IITP-MSIT under Grant \mbox{RS-2024-00436680} (Collaborative Research Projects with Microsoft Research) through Global Research Support Program in the Digital Field program, in part by KIAT under Grant \mbox{P0020536} (HRD Program for Industrial Innovation), in part by MOTIR under Grant \mbox{RS-2025-25462891}, in part by MOTIR under Grant \mbox{RS-2025-25460896}, in part by SNU-IPAI, in part by SNU-AIIS, in part by SNU-IAMD, in part by SNU BK21+ Program in Mechanical Engineering, in part by SNU Institute for Engineering Research, in part by Hyundai Motor Company and Kia, and in part by Microsoft Research Asia.}

\bibliography{references}

\newpage

\appendix
\newtheorem{proposition}{Proposition}[section]
\newtheorem{lemma}[proposition]{Lemma}
\newtheorem{corollary}[proposition]{Corollary}
\newenvironment{proof}{\par\smallskip\noindent\textit{Proof.}\hspace*{0.5em}}{\nobreak\hfill$\square$\par\smallskip}

\section{Operational-Space Dynamics and Block-Isotropic Inertia}
\label{app:dynamics}

This appendix derives the commanded wrench and its feedforward/feedback split (Appendix~\ref{app:e_dynamics}), the block-isotropic inertia surrogate $\bar\Lambda$ that sets the metric for our projectors and gains (Appendix~\ref{app:normalized_block_iso}), and the body-frame tracking errors $(e_T,e_V)$ entering the feedback wrench (Appendix~\ref{app:tracking_errors}).
Together they back Sections~\ref{sec:bg_impedance} and~\ref{sec:gain_decomposition} and supply the plant for the passivity argument of Appendix~\ref{app:e_passivity}.

\subsection{Commanded Wrench from the Virtual Impedance Model}
\label{app:e_dynamics}

We control the grasped end-effector in its body frame $\{b\}$ with body twist $V_b=(\omega,v)\in\se(3)\cong\Rbb^6$.
In body coordinates the operational-space dynamics read
\begin{equation}
    \Lambda_b\, \dot V_b + C_b\, V_b + g_b = F_{\rm cmd} - F_{\rm ext},
    \label{eq:appA:osc_dynamics}
\end{equation}
with operational-space inertia $\Lambda_b(q)$, Coriolis/centrifugal $C_b(V_b,q)$, gravity $g_b(q)$, commanded wrench $F_{\rm cmd}$ realized by $\tau=J^\top F_{\rm cmd}$ (cf.\ Eq.~\eqref{eq:bg_commanded_wrench}), and contact wrench $F_{\rm ext}$ expressed in $\{b\}$.
We suppress the configuration argument below and restore $\Lambda_b(q)$ only where its dependence is essential (Appendix~\ref{app:normalized_block_iso}).

The virtual impedance target of Section~\ref{sec:bg_impedance} is
\begin{equation}
    M_v \dot e_V + D_v e_V + K_v e_T = F_{\rm ext},
    \label{eq:appA:target_impedance}
\end{equation}
with the body-frame errors $e_T,e_V$ of Appendix~\ref{app:tracking_errors}.
The desired body twist $V_d$ is represented in the desired frame $\{d\}$; the adjoint transports it into $\{b\}$,
\begin{equation}
    V_{\rm ref,b}:=\Ad_{T_b^{-1}T_d}V_d,
    \qquad
    a_{\rm ref,b}:=\dot V_{\rm ref,b}
    =\tfrac{d}{dt}\!\left(\Ad_{T_b^{-1}T_d}V_d\right),
    \label{eq:appA:ref_accel}
\end{equation}
so $a_{\rm ref,b}$ includes the derivative of the adjoint in addition to the transported desired acceleration.
With $e_V=V_{\rm ref,b}-V_b$ (hence $\dot e_V=a_{\rm ref,b}-\dot V_b$) and the matched virtual inertia $M_v=\Lambda_b$, the target~\eqref{eq:appA:target_impedance} becomes
\begin{equation}
    \Lambda_b \dot V_b
    = \Lambda_b a_{\rm ref,b} + D_v e_V + K_v e_T - F_{\rm ext},
    \label{eq:appA:desired_accel}
\end{equation}
while the plant~\eqref{eq:appA:osc_dynamics} gives $\Lambda_b \dot V_b = F_{\rm cmd} - F_{\rm ext} - C_b V_b - g_b$.
Equating the two yields the commanded wrench in the nominal model,
\begin{equation}
    F_{\rm cmd}
    =
    \Lambda_b a_{\rm ref,b} + C_b V_b + g_b + K_v e_T + D_v e_V .
    \label{eq:appA:Fcmd_full}
\end{equation}
A model-based implementation replaces the true plant terms by model estimates $\hat\Lambda_b,\hat C_b,\hat g_b$ and groups the result into a model-based feedforward and an impedance feedback,
\begin{equation}
    F_{\rm ff} = \hat\Lambda_b\, a_{\rm ref,b} + \hat C_b\, V_b + \hat g_b ,
    \qquad
    F_{\rm fb} = K_v e_T + D_v e_V,
    \qquad
    F_{\rm cmd}=F_{\rm ff}+F_{\rm fb} .
    \label{eq:appA:Fff}
\end{equation}
With exact model matching, substituting~\eqref{eq:appA:Fff} into~\eqref{eq:appA:osc_dynamics} recovers the target~\eqref{eq:appA:target_impedance}.
Because $M_v=\Lambda_b$ cancels $F_{\rm ext}$ between target and plant, $F_{\rm cmd}$ carries \emph{no} contact-wrench term: the policy, subspace estimator, and impedance feedback use no force/torque feedback, which is used only for safety monitoring and evaluation (Appendix~\ref{app:impl}).

\subsection{Block-Isotropic and Normalized Inertia}
\label{app:normalized_block_iso}

Appendix~\ref{app:e_dynamics} fixed $M_v=\Lambda_b$ but left the gain structure open.
The ideal configuration-dependent controller factors stiffness and damping through the same operational-space inertia,
\begin{equation}
    K_v(q)=\Lambda_b(q)K_{\rm diag},
    \qquad
    D_v(q)=\Lambda_b(q)D_{\rm diag},
    \label{eq:appA:khatib_gains}
\end{equation}
with diagonal positive $K_{\rm diag},D_{\rm diag}$, reducing the target to the inertia-normalized error dynamics
\begin{equation}
    \dot e_V+D_{\rm diag}e_V+K_{\rm diag}e_T=\Lambda_b(q)^{-1}F_{\rm ext}
    \label{eq:appA:khatib_error}
\end{equation}
of the classical operational-space control construction of Khatib, with the wrench entering through the mobility $\Lambda_b(q)^{-1}$.
This is why our gains carry a leading inertia factor: $K_v,D_v$ are wrench-level maps that must be consistent with the metric measuring twist errors.

Many implementations instead keep the gains fixed, e.g.\ a constant block with isotropic rotational/translational sub-blocks,
\begin{equation}
    K_v=\operatorname{blkdiag}(k_{\rm rot}I_3,\ k_{\rm trans}I_3),
    \qquad
    D_v=\operatorname{blkdiag}(d_{\rm rot}I_3,\ d_{\rm trans}I_3),
    \label{eq:appA:fixed_block_gain}
\end{equation}
robust, but unable to decide which task directions should be stiff (driving) and which compliant (constrained).
We keep the fixed-gain philosophy but make the inertia weighting explicit.
Using the configuration-dependent $\Lambda_b(q)$ as the projector metric would tie the feasible/constrained split to an online inertia estimate, defeating the single-shot design, so we replace it by a constant block-isotropic surrogate
\begin{equation}
    \bar\Lambda = \operatorname{blkdiag}(\bar I\, I_3,\ \bar m\, I_3) \succ 0 ,
    \label{eq:appA:block_iso}
\end{equation}
with single rotational and translational scales $\bar I,\bar m>0$.
The gains of Eq.~\eqref{eq:decomposed_gains} are then fixed during execution yet retain the rotational-to-translational ratio through $\bar\Lambda$.
Two distinctions matter here: $\bar\Lambda$ is \emph{neither} the feedforward inertia $\hat\Lambda_b$ of Appendix~\ref{app:e_dynamics} (which compensates the plant), nor the PCA metric $G=\operatorname{blkdiag}(\alpha^2I_3,I_3)$ of Appendix~\ref{app:uncentered} (a data-driven scaling on sampled twists).
It is the inertia metric linking the twist split to the wrench split (Appendices~\ref{app:energy_decomp},~\ref{app:closed_loop},~\ref{app:gain_parity}).

\paragraph{Computing $\bar\Lambda$.}
All quantities are in the control body frame $\{b\}$, so $\Lambda_b$ is referred to the control-frame origin.
We partition $\Lambda_b=\left[\begin{smallmatrix}\Lambda_{\omega\omega}&\Lambda_{\omega v}\\\Lambda_{v\omega}&\Lambda_{vv}\end{smallmatrix}\right]$ and collapse each diagonal block to its isotropic trace average,
\begin{equation}
    \bar I := \tfrac{1}{3}\operatorname{tr}(\Lambda_{\omega\omega}),
    \qquad
    \bar m := \tfrac{1}{3}\operatorname{tr}(\Lambda_{vv}),
    \qquad
    \ell_\Lambda := \sqrt{\bar I/\bar m},
    \label{eq:appA:lambda_trace_average}
\end{equation}
discarding within-block anisotropy and angular--linear coupling from the projector metric.
Their effect stays in the true closed-loop dynamics, and we do not claim exact decoupling.
In simulation $\Lambda_b$ is the free gripper's spatial inertia about the control point; on FR3 it is $\Lambda_b(q_{\rm nominal})=(J_bM^{-1}J_b^\top)^{-1}$ at the nominal pose with payload.
The per-platform snapshots (not universal constants) are reported in Appendix~\ref{app:numerical_gains}.

\paragraph{Scale normalization.}
Writing $\bar\Lambda=\bar m\,\operatorname{blkdiag}(\ell_\Lambda^2 I_3,I_3)$, the absolute scale $\bar m$ is only an overall factor: $\bar\Lambda\mapsto c\bar\Lambda$ ($c>0$) leaves the projectors of Eq.~\eqref{eq:lambda_projectors} unchanged and multiplies the gains of Eq.~\eqref{eq:decomposed_gains} by $c$ (undone by rescaling $k_i,d_i$ by $1/c$).
Only the length ratio $\ell_\Lambda$ fixes the projector geometry; $\bar m$ is absorbed into the scalar gains, reported with $\bar I$ in Appendix~\ref{app:numerical_gains}.
In all experiments, the projectors and gains use the normalized metric $(m_{\rm ref}/\bar m)\,\bar\Lambda=m_{\rm ref}\operatorname{blkdiag}(\ell_\Lambda^2 I_3,I_3)$ with $m_{\rm ref}=1$~kg, and the scalar gains of Appendix~\ref{app:numerical_gains} refer to this metric.

\subsection{Body-Frame Tracking Error Definitions}
\label{app:tracking_errors}

Since the decomposed gains act in body-twist coordinates, the tracking errors are defined in the same coordinates.
We use the body-frame group error of Bullo and Murray~\cite{BulloMurray1995, BulloMurray1999}---a double-geodesic pose error with a compatible adjoint-transported twist error.
With our reference-minus-current convention, Eq.~\eqref{eq:tracking_errors} is
\begin{equation}
    e_T =
    \begin{bmatrix}
    \log(R_b^\top R_d)^\vee \\[2pt]
    R_b^\top(p_d-p_b)
    \end{bmatrix},
    \qquad
    e_V = \Ad_{T_b^{-1}T_d}\, V_d - V_b ,
    \label{eq:appA:tracking_errors}
\end{equation}
with $T_b=(R_b,p_b)$, $T_d=(R_d,p_d)$, and desired body twist $V_d$.
The pose coordinate is ``double-geodesic'' rather than the full $\SE(3)$ logarithm: a geodesic $\SO(3)$ axis-angle block $\log(R_b^\top R_d)^\vee$ and a Euclidean displacement $R_b^\top(p_d-p_b)$ resolved in the current body frame, so both blocks share the body frame of the sampled twists, the projectors of Eq.~\eqref{eq:lambda_projectors}, and the gains of Eq.~\eqref{eq:decomposed_gains}~\cite{lynch2017modern}.
The twist error transports $V_d$ by the adjoint before subtraction (Appendix~\ref{app:e_dynamics}); the reference-minus-current sign matches Eq.~\eqref{eq:bg_commanded_wrench}, so a positive error produces the restoring wrench $F_{\rm fb}=K_v e_T+D_v e_V$.

For small errors this reduces to the standard impedance law: with $T_b^{-1}T_d=\exp(\xi^\wedge)$ and $\xi$ small, $e_T=\xi+O(\|\xi\|^2)$ and $\dot e_T=e_V+O(\|e_T\|\,\|e_V\|)$, so the feedback realizes Eq.~\eqref{eq:bg_target_impedance} with $M_v=\Lambda_b$.
Appendix~\ref{app:closed_loop} develops the exact local error kinematics and the frozen first-order conditions supporting the passivity certificate.

\section{Configuration Invariance from Lie Subalgebra}
\label{app:lie_subalgebra}

This appendix supports the kinematic classification of Section~\ref{sec:bg_kinematics}: a sufficient condition for the feasible body-twist subspace $\mathcal S(q)$ to be configuration-independent, and its verification for the common joint primitives.
These are reference calculations, not controller inputs: the estimator of Section~\ref{sec:constraint_estimation} observes sampled twists without assuming DoF, joint types, screw axes, or an articulation model.
Configuration-dependent mechanisms (crank, universal joint) are treated in Appendix~\ref{app:envelope_examples}.

Recall that $\mathcal S(q)=\operatorname{span}\{B_1(q),\ldots,B_k(q)\}\subset\se(3)$, Eq.~\eqref{eq:body_twist_kinematics}, under the identification $\se(3)\cong\Rbb^6$ via $V=(\omega,v)$; we take the $\se(3)$ view because the argument uses the bracket structure.
A subspace is a \emph{Lie subalgebra} iff it is closed under the Lie bracket, and by bilinearity it suffices to check the generators, $[B_i,B_j]\in\mathcal S$.
For $\se(3)$ the bracket of $V_1=(\omega_1,v_1)$ and $V_2=(\omega_2,v_2)$ is~\cite{lynch2017modern}
\begin{equation}
    [V_1, V_2] =
    \begin{bmatrix}
        \omega_1 \times \omega_2 \\
        \omega_1 \times v_2 - \omega_2 \times v_1
    \end{bmatrix}.
    \label{eq:appB:se3_bracket}
\end{equation}

\begin{proposition}[Lie-subalgebra condition]
\label{prop:appB:lie_invariance}
Let $\{B_1,\ldots,B_k\}$ be the home body screw axes of a serial open-chain product-of-exponentials model, and let $\mathcal S=\operatorname{span}\{B_1,\ldots,B_k\}$.
If $\mathcal S$ is a Lie subalgebra of $\se(3)$, then the body-Jacobian columns satisfy $B_i(q)\in\mathcal S$ for every $q$ in the joint-coordinate domain.
Hence $\mathcal S(q)\subseteq\mathcal S$ for all such $q$.
If, over the configuration range of interest, the body Jacobian has rank $\dim\mathcal S$, then $\mathcal S(q)=\mathcal S$ throughout that range; the object is configuration-independent on that regular set.
\end{proposition}

\begin{proof}
The proof has two steps.
First, each body-Jacobian column at configuration $q$ is an adjoint image of a home screw axis.
Second, because adjoint transport by an element generated inside a Lie subalgebra preserves that subalgebra, every transported column remains inside $\mathcal S$.

Each body-Jacobian column is the corresponding home screw transported by the more distal joint motions, so the columns at $q=(q_1,\ldots,q_k)$ are~\cite{lynch2017modern}
\begin{equation}
    B_i(q) = \Ad_{e^{-B_k^\wedge q_k} \cdots e^{-B_{i+1}^\wedge q_{i+1}}}\, B_i,
    \qquad B_k(q)=B_k.
    \label{eq:appB:body_jacobian}
\end{equation}
The Lie-group identity $\Ad_{e^{X^\wedge}} V = e^{\operatorname{ad}_X} V = V+[X,V]+\tfrac{1}{2!}[X,[X,V]]+\cdots$ (with $\operatorname{ad}_X V=[X,V]$) expands the adjoint as a series of iterated brackets.
Setting $X=-B_j q_j\in\mathcal S$ (as $\mathcal S$ is a vector space) for a single distal factor and taking any $V\in\mathcal S$, every term is an iterated bracket of elements of $\mathcal S$.
Because $\mathcal S$ is closed under the bracket, every term lies in $\mathcal S$, hence so does the sum:
\[
    \Ad_{e^{-B_j^\wedge q_j}} V \in \mathcal S
    \qquad\text{for all } V\in\mathcal S.
    \qquad(\mathcal S\text{ bracket-closed})
\]
Because $\Ad_{AB}=\Ad_A\Ad_B$, the transport in Eq.~\eqref{eq:appB:body_jacobian} is a product of such factors, and applying the displayed fact to each factor in turn yields $B_i(q)\in\mathcal S$ for all $i$ and all $q$ in the joint-coordinate domain.
Therefore $\mathcal S(q)=\operatorname{span}\{B_i(q)\}_{i=1}^{k}\subseteq\mathcal S$.
On any regular configuration range where $\operatorname{rank}J_b(q)=\dim\mathcal S$, the subspace $\mathcal S(q)$ has the same dimension as $\mathcal S$, so the inclusion is an equality.
At singular configurations the inclusion still holds, but the instantaneous feasible subspace may drop rank.
\end{proof}

We verify the condition for each configuration-independent primitive used in the paper, treating each as the chain of its generating body screws (Table~\ref{tab:appB:primitives}).
All screws are in the home body frame; changing that frame applies a fixed $\SE(3)$ adjoint to every generator and hence to the span, and the rank condition of Proposition~\ref{prop:appB:lie_invariance} holds away from the joint chart's singularities.
The $1$-DoF cases (revolute, prismatic, helical) and the cylindrical pair are Abelian: a $1$-D span self-commutes, the cylindrical generators share an axis ($[(\hat n;-\hat n\times p),(0;\hat n)]=0$), and pure translation of any dimension has zero angular part.
The only nontrivial brackets are rotational, illustrated by the spherical joint.

\begin{table}[H]
\centering
\caption{\textbf{Configuration-independent primitives and their Lie-subalgebra verification.}
Body screws $B=(\omega;v)$ in the home frame: unit axes $\hat n,\hat d$, axis point $p$, helical pitch $h$, in-plane unit vectors $\hat t_1,\hat t_2\perp\hat n$.
Every generator bracket lies in the span, so each is a Lie subalgebra and, by Proposition~\ref{prop:appB:lie_invariance}, configuration-independent on its regular set.}
\label{tab:appB:primitives}
\vspace{0.25em}
\small
\setlength{\tabcolsep}{6pt}
\begin{tabular}{l c l l}
\toprule
Primitive & DoF & Generators & Structure \\
\midrule
Revolute (R)    & $1$ & $(\hat n;\,-\hat n\times p)$ & Abelian \\
Prismatic (P)   & $1$ & $(0;\,\hat d)$ & Abelian (any dim) \\
Helical (H)     & $1$ & $(\hat n;\,h\hat n-\hat n\times p)$ & Abelian \\
Cylindrical (C) & $2$ & $(\hat n;\,-\hat n\times p),\ (0;\,\hat n)$ & Abelian \\
Spherical (S)   & $3$ & $(\hat e_i;\,-\hat e_i\times p)_{i=1,2,3}$ & $\so(3)$ \\
Planar (E)      & $3$ & $(\hat n;\,0),\ (0;\,\hat t_1),\ (0;\,\hat t_2)$ & $\se(2)$ \\
\bottomrule
\end{tabular}
\end{table}

\paragraph{Spherical (S), worked.}
For three rotations $B_i=(\hat e_i;-\hat e_i\times p)$ about a common point $p$, the rotational block of the bracket is $\hat e_1\times\hat e_2=\hat e_3$ and the triple product $a\times(b\times c)=b(a\cdot c)-c(a\cdot b)$ collapses the translational block to $-\hat e_3\times p$:
\begin{equation}
    [B_1, B_2]
    =
    \begin{bmatrix}
        \hat e_1\times\hat e_2 \\
        \hat e_1\times(-\hat e_2\times p)-\hat e_2\times(-\hat e_1\times p)
    \end{bmatrix}
    =
    \begin{bmatrix}
        \hat e_3 \\
        -\hat e_2 p_1+\hat e_1 p_2
    \end{bmatrix}
    =
    \begin{bmatrix}
        \hat e_3 \\
        -\hat e_3\times p
    \end{bmatrix}
    = B_3 \in \mathcal S_S .
    \label{eq:appB:sph_bracket}
\end{equation}
By cyclic symmetry $[B_2,B_3]=B_1$ and $[B_3,B_1]=B_2$, so $\mathcal S_S\cong\so(3)$.
The planar joint closes analogously, $[B_1,B_2]=\pm B_3$ (sign fixed by the in-plane frame orientation) and $[B_2,B_3]=0$, giving $\mathcal S_E\cong\se(2)$; shifting the origin by $p$ only adds the in-plane translation $-\hat n\times p$, a basis change within the same span.

By contrast, multi-joint mechanisms whose body-Jacobian span is not bracket-closed, or whose adjoint action carries a joint screw outside the initial span, need not be configuration-independent; Appendix~\ref{app:envelope_examples} works through such cases.

\section{Configuration Dependence and the Global Envelope}
\label{app:envelope_examples}

When $\mathcal S(q)$ varies with $q$, the fixed target is the global envelope $\mathcal S_{\rm env}$ of Eq.~\eqref{eq:subspace_envelope}.
We compute it for three nondegenerate $2$R linkages---the simulation universal joint (Section~\ref{sec:exp_sim}) and the coffee-grinder crank of Figure~\ref{fig:articulated_examples}(b) among them---collected in Table~\ref{tab:appC:envelope_examples}.
As in Appendix~\ref{app:lie_subalgebra}, these subspaces are reference targets only.

We use a distal frame $\{c\}$ with origin on the distal joint axis, so $B_2$ has zero translational part and stays constant while the proximal screw $B_1(q_2)=\Ad_{e^{-B_2^\wedge q_2}}B_1$ varies with the distal angle~\cite{lynch2017modern}; a displaced gripper frame $\{b\}$ applies the fixed adjoint $\Ad_{T_{bc}}$ ($T_{bc}$ the pose of $\{c\}$ in $\{b\}$) to every screw and to $\mathcal S_{\rm env}$.
At each fixed $q$ the local subspace $\mathcal S(q)=\operatorname{span}\{B_1(q_2),B_2\}$ is $2$-dimensional, but sweeping $q_2$ over an interval enlarges $\mathcal S_{\rm env}$; the dimension claims assume distinct axes, $L>0$ where an offset is present, and a distal range with nonempty interior.
We abbreviate $\omega_i=(\hat e_i;0)$, $v_i=(0;\hat e_i)$.

\paragraph{U-joint (worked example).}
Two perpendicular axes meet at a point; placing the $\{c\}$-origin at the intersection eliminates every translational component, so the envelope is purely rotational:
\begin{equation}
    B_2 = (\hat e_x;\, 0), \qquad
    B_1(q_2) = \big(R_x(-q_2)\,\hat e_y;\; 0\big),
    \qquad
    \mathcal S_{\rm env} = \operatorname{span}\{\omega_x,\,\omega_y,\,\omega_z\}.
    \label{eq:appC:ujoint}
\end{equation}
Here $\dim\mathcal S_{\rm env}=3$ while $\dim\mathcal S(q)=2$ at each fixed $q$: as $q_2$ sweeps, $R_x(-q_2)\hat e_y$ traces the $\omega_y$--$\omega_z$ plane, which with the constant $\omega_x$ of $B_2$ fills the full rotation space.
This is the universal-joint object of our simulation study (Section~\ref{sec:exp_sim}): local feasible dimension $2$ at every configuration, with a fixed $3$-dimensional global envelope as the target our method estimates.

\paragraph{Parallel-offset crank (worked example).}
Both axes are along $\hat z$; with the $\{c\}$-origin on the handle axis and the proximal axis offset by $L\hat e_x$ at $q_2=0$,
\begin{equation}
    B_2 = (\hat e_z;\, 0), \qquad
    B_1(q_2) = \big(\hat e_z;\; -L\sin q_2\,\hat e_x - L\cos q_2\,\hat e_y\big),
    \label{eq:appC:crank_screws}
\end{equation}
the difference $B_1(q_2)-B_2$ is a pure translation sweeping the $\hat x$--$\hat y$ plane, so
\begin{equation}
    \mathcal S_{\rm env} = \operatorname{span}\{\omega_z,\,v_x,\,v_y\},
    \qquad \dim\mathcal S_{\rm env} = 3 \quad (\dim\mathcal S(q)=2).
    \label{eq:appC:crank_env}
\end{equation}
Setting $L=l$ recovers the body screws of Section~\ref{sec:bg_kinematics}, $B_1(q)=(0,0,1,-l\sin q_2,-l\cos q_2,0)^\top$ and $B_2=(0,0,1,0,0,0)^\top$: $\mathcal S(q)$ rotates within the fixed plane $\mathcal S_{\rm env}$, as Figure~\ref{fig:articulated_examples}(b) shows.
The door-with-lever (off-axis orthogonal, $\dim\mathcal S_{\rm env}=4$, not an experimental object here) follows the same construction; Table~\ref{tab:appC:envelope_examples} collects all three mechanisms' screws and envelopes.

\begin{table}[H]
\centering
\caption{\textbf{Global envelope for the three nondegenerate configuration-dependent $2$R examples.}
Screws are expressed in the distal frame $\{c\}$ with origin on the distal joint axis, so $B_2$ is constant and $B_1(q_2)=\Ad_{e^{-B_2^\wedge q_2}}B_1$.
At every fixed $q$ the local feasible subspace has $\dim\mathcal S(q)=2$; sweeping $q_2$ over an interval with nonempty interior enlarges the envelope to the listed dimension.
Shorthand: $\omega_i=(\hat e_i;0)$, $v_i=(0;\hat e_i)$; ``$\geq k$ phases'' means $k$ distinct distal angles $q_2$.}
\label{tab:appC:envelope_examples}
\vspace{0.25em}
\small
\setlength{\tabcolsep}{5pt}
\resizebox{\linewidth}{!}{%
\begin{tabular}{l c c c c l}
\toprule
Mechanism & $B_2$ & $B_1(q_2)$ & $\mathcal S_{\rm env}$ & $\dim\mathcal S_{\rm env}$ & Excitation for full $\mathcal S_{\rm env}$ \\
\midrule
U-joint (intersecting 2R)
  & $(\hat e_x;0)$ & $(R_x(-q_2)\hat e_y;\,0)$
  & $\operatorname{span}\{\omega_x,\omega_y,\omega_z\}$ & $3$
  & $B_2$ excited, $\geq2$ phases \\
Parallel-offset crank (parallel 2R)
  & $(\hat e_z;0)$ & $(\hat e_z;\,-L\sin q_2\,\hat e_x-L\cos q_2\,\hat e_y)$
  & $\operatorname{span}\{\omega_z,v_x,v_y\}$ & $3$
  & $\geq3$ $B_1$ phases ($B_2$ redundant) \\
Door hinge $+$ lever (off-axis orthogonal 2R)
  & $(\hat e_x;0)$ & $(R_x(-q_2)\hat e_z;\,-L\,\hat e_x)$
  & $\operatorname{span}\{\omega_x,\omega_y,\omega_z,v_x\}$ & $4$
  & $B_2$ excited, $\geq3$ phases \\
\bottomrule
\end{tabular}%
}
\end{table}

\paragraph{Excitation.}
Each $B_1(q_2)$ depends on $q_2$ only through $\{1,\sin q_2,\cos q_2\}$, independent on any interval, so the continuous sweep gives $\operatorname{span}(\{B_2\}\cup\{B_1(q_2):q_2\in I\})=\mathcal S_{\rm env}$ and a finite ensemble $V_j=\dot q_{1j}B_1(q_{2j})+\dot q_{2j}B_2$ recovers it once $\operatorname{rank}[V_1,\ldots,V_M]=\dim\mathcal S_{\rm env}$.
The last column of Table~\ref{tab:appC:envelope_examples} states the requirement: excite $B_2$ where it adds a direction $B_1(q_2)$ never reaches (the constant $\omega_x$ for U-joint and door; redundant for the crank), and visit enough distinct distal phases to span the swept block.
A narrow configuration range or a near-one-dimensional joint-velocity pattern collapses the empirical span---the configuration-dependent counterpart of the sufficient condition of Appendix~\ref{app:lie_subalgebra}.

\paragraph{Feasibility versus envelope membership.}
For a configuration-independent object $\mathcal S(q)=\mathcal S_{\rm env}$ throughout the regular range, so $V_b$ is feasible \emph{iff} $V_b\in\mathcal S_{\rm env}$.
For the configuration-dependent examples feasibility is strictly stronger: $V_b$ feasible $\Rightarrow V_b\in\mathcal S(q)\subsetneq\mathcal S_{\rm env}$, but not conversely---for the crank both $v_x,v_y\in\mathcal S_{\rm env}$ while only one in-plane direction is instantaneously feasible.
Thus $\mathcal S_{\rm env}$ is a \emph{conservative outer bound}: it never excludes a feasible direction but can include momentarily infeasible ones, so treating the whole envelope as stiff leaves residual stiffness along instantaneously constrained directions.
This is one kinematic reason for the smaller improvement on the crank (Section~\ref{sec:exp_real}).
In both classes $\mathcal S_{\rm env}$ is a fixed target, motivating its direct estimation from sampled twists (Appendix~\ref{app:uncentered}).

\section{Trajectory-Ensemble Subspace Estimation in the PCA Metric}
\label{app:uncentered}

Appendix~\ref{app:envelope_examples} established $\mathcal S_{\rm env}$ as a fixed linear subspace of $\se(3)$ through the zero twist.
We estimate it from a finite, noisy ensemble of body twists by the weighted PCA of Section~\ref{sec:constraint_estimation}.
Estimating a subspace through the origin requires three choices: the inner product defining a \emph{principal} direction, since angular and linear twist components carry different units (Appendix~\ref{app:pca_svd}); the characteristic length $\alpha$ setting the angular-to-linear scale (Appendix~\ref{app:pca_scaled}); and the reference point of the second moment, taken at the origin rather than the sample mean because the target passes through zero (Appendix~\ref{app:uncentered_centered}).

\subsection{Why a PCA Metric Is Required}
\label{app:pca_svd}

Let $W=[V_1,\ldots,V_M]\in\Rbb^{6\times M}$, $M=NH$, be the body-twist ensemble.
Ordinary Euclidean PCA on the stacked coordinates would declare one unit of angular velocity comparable to one unit of linear velocity---an arbitrary choice, since changing the length unit or a lever-arm convention rotates the principal components though the motion is unchanged.
PCA on twists is meaningful only under a chosen inner product on $\se(3)$~\cite{park1995distance}.
We fix a symmetric positive-definite PCA metric $G$ on twists, distinct from the inertia metric $\bar\Lambda$ of Appendix~\ref{app:normalized_block_iso} that defines the gains in Appendix~\ref{app:energy_decomp}, and define the leading principal direction by the uncentered $G$-variance,
\begin{equation}
    \langle V_1,V_2\rangle_G := V_1^\top G V_2,
    \qquad
    u_1 = \arg\max_{u^\top G u=1}\frac{1}{M}\sum_{\ell=1}^{M}\langle u,V_\ell\rangle_G^2 ,
    \label{eq:appD:metric_pca_obj}
\end{equation}
with higher directions under the same $G$-orthogonality constraint (the origin reference is justified in Appendix~\ref{app:uncentered_centered}).
The metric thus determines which angular--linear mixtures count as large variation.

The computation is an SVD.
With the Cholesky factor $G=L^\top L$, $\langle V_1,V_2\rangle_G=(LV_1)^\top(LV_2)$, so $G$-metric PCA is ordinary PCA on the transformed data,
\begin{align}
LW &= \widetilde U\,\widetilde\Sigma\,\widetilde V^\top,
&&\text{(SVD of metric-transformed twists)} \label{eq:appD:svd}\\
U &= L^{-1}\widetilde U,
&&\text{(back-transform to }\se(3)\text{)} \label{eq:appD:backtransform}
\end{align}
and the columns are $G$-orthonormal, $\langle u_i,u_j\rangle_G=\widetilde u_i^\top\widetilde u_j=\delta_{ij}$.
The captured uncentered $G$-variances are $\lambda_j=\widetilde\sigma_j^2/M$; in the noiseless rank-$m$ case only the first $m$ are nonzero.
Real ensembles have all six singular values positive, so we select $m$ by the rule of Section~\ref{sec:constraint_estimation}, gating each direction against an absolute speed $v_{\rm abs}$ and the least-excited mode,
\begin{equation}
    m=\#\bigl\{k: \sqrt{\lambda_k}>\max(\kappa_{\rm rel}\sqrt{\lambda_6},\,v_{\rm abs})\bigr\},
    \qquad
    \widehat{\mathcal S}_{\rm env}=\operatorname{span}\{u_1,\ldots,u_m\}=\operatorname{range}U_\parallel .
    \label{eq:appD:dimsel}
\end{equation}
Because $G$ weights the angular block by $\alpha^2$, the scaled twist $(\alpha\omega,v)$ has units of linear velocity at the grasp point (Appendix~\ref{app:pca_scaled}), so $\sqrt{\lambda_k}$, the root-mean-square component of the scaled twists along the $k$-th principal direction $\widetilde u_k$, is a speed in this metric, and $v_{\rm abs}$ a physical speed.
Each test targets one of the two sources of off-manifold components considered in Section~\ref{sec:constraint_estimation}.
The absolute test bounds the residual speed left by the generation limit of the policy; it does not depend on the least-excited mode, so it remains reliable when $\widetilde\sigma_6$ is small, as in simulation, where the grasp is rigid.
The relative test ties the cutoff to the spectral gap above the least-excited mode, whose level reflects grasp-chain deviations.
We apply both tests in both domains, with values set according to the main source in each (Appendix~\ref{app:hyperparams}).

\subsection{Meaning of the Characteristic Length \texorpdfstring{$\alpha$}{alpha}}
\label{app:pca_scaled}

Our block-isotropic metric and its Cholesky factor are
\begin{equation}
    G=\operatorname{blkdiag}(\alpha^2 I_3, I_3),
    \qquad
    L_\alpha=\operatorname{blkdiag}(\alpha I_3, I_3),
    \qquad
    L_\alpha V=(\alpha\omega,v),
    \label{eq:appD:Galpha}
\end{equation}
so $\alpha$ compares angular velocity to the linear velocity it would induce at a lever arm of length $\alpha$: small $\alpha$ makes rotation cheap, large $\alpha$ expensive, and the SVD of $L_\alpha W$ is PCA on the scaled twists $(\alpha\omega,v)$, back-transformed by~\eqref{eq:appD:backtransform}.
We set $\alpha$ from the unweighted ensemble in the controller's body coordinates and SI units (it is not an invariant object parameter), using the leading left singular vector $u=[u_\omega^\top,u_v^\top]^\top$ of $W$:
\begin{equation}
    \alpha=\operatorname{clip}\!\big(\|u_v\|/\|u_\omega\|,\ \alpha_{\min},\ \alpha_{\max}\big).
    \label{eq:appD:alpha}
\end{equation}
The ratio $\|u_v\|/\|u_\omega\|$ is the dominant translation-per-rotation scale; it equalizes the leading mode under $L_\alpha$ ($\|\alpha u_\omega\|=\|u_v\|$), so the dominant motion is neither rotation- nor translation-dominated in the PCA metric.
Clipping prevents nearly pure-rotation or pure-translation samples from ill-conditioning the metric (pure translation, $\|u_\omega\|=0$, clips to $\alpha_{\max}$).
This $\alpha$ is not the inertia length $\ell_\Lambda=\sqrt{\bar I/\bar m}$ of Appendix~\ref{app:normalized_block_iso}: $\alpha$ scales the data \emph{before} estimation, whereas $\ell_\Lambda$ defines the control metric applied \emph{after} the subspace is fixed.

\subsection{Why the PCA Is Uncentered}
\label{app:uncentered_centered}

The target $\mathcal S_{\rm env}$ passes through the zero twist---zero is feasible for every object and feasible motions are closed under scaling, so each $\mathcal S(q)$ and their span pass through the origin.
A centered covariance instead anchors at the sample mean and measures only deviations about the average motion, the wrong reference here.
With $\bar V=M^{-1}\sum_\ell V_\ell$, $W_c=W-\bar V\mathbf 1^\top$, and $G=L_\alpha^\top L_\alpha$, the uncentered and centered second moments differ by a rank-one mean term,
\begin{equation}
    \widehat M_0=\tfrac{1}{M}L_\alpha W W^\top L_\alpha^\top,
    \quad
    \widehat M_c=\tfrac{1}{M}L_\alpha W_c W_c^\top L_\alpha^\top,
    \quad
    \widehat M_0-\widehat M_c=L_\alpha\bar V\bar V^\top L_\alpha^\top ,
    \label{eq:appD:uncentered_centered_def}
\end{equation}
so centering removes exactly the consistent task-progress direction.
The cost is sharpest for a near-$1$-DoF ensemble $V_\ell=V_0+\varepsilon_\ell$ with $V_0\in\mathcal S_{\rm env}$ and $\|\varepsilon_\ell\|\ll\|V_0\|$: the uncentered moment scales with $\|V_0\|^2$ and recovers $V_0$, while the centered covariance subtracts $V_0$ and is dominated by the residual.
For articulated manipulation deterministic progress is signal, not bias.
In the noiseless factorization $W=U_\star Z$ ($\operatorname{range}U_\star=\mathcal S_{\rm env}$), the uncentered estimator recovers $\mathcal S_{\rm env}$ iff $\operatorname{rank}Z=\dim\mathcal S_{\rm env}$, in which case $\operatorname{range}U_\parallel=\operatorname{range}U_\star$; with off-subspace perturbation, recovery depends on the spectral gap gated by Eq.~\eqref{eq:appD:dimsel}.
The resulting basis $U_\parallel$ feeds the projector and gain construction of Appendix~\ref{app:energy_decomp}.

\section{Inertia-Metric Projectors and Energy-Consistent Gains}
\label{app:energy_decomp}

This appendix verifies two algebraic facts behind the gain construction: the Gram-matrix formula is the correct $\bar\Lambda$-orthogonal projector even though the PCA basis $U_\parallel$ is not $\bar\Lambda$-orthonormal, and the stiffness/damping matrices are Hessians of a quadratic potential and a Rayleigh dissipation in the local error coordinates $e_T,e_V\in\Rbb^6$, so their symmetry and positive (semi)definiteness follow rather than being assumed.
The closed-loop passivity they support is proved separately, under a local small-error scope, in Appendix~\ref{app:e_passivity}.
Throughout, $\bar\Lambda\succ0$ is the inertia metric of Appendix~\ref{app:normalized_block_iso}.

\subsection{\texorpdfstring{$\bar\Lambda$}{Lambda-bar}-Orthogonal Twist Projectors}
\label{app:c_projectors}

Let $U_\parallel\in\Rbb^{6\times m}$ be any full-rank basis of the estimated subspace; it is generally not $\bar\Lambda$-orthonormal, so the Gram matrix in the projector formula is essential.

\begin{proposition}[$\bar\Lambda$-orthogonal projector properties]
\label{prop:appE:proj}
Let $A:=U_\parallel^\top\bar\Lambda U_\parallel\in\Rbb^{m\times m}$ and define
\begin{equation}
    P_\parallel^{\bar\Lambda}=U_\parallel A^{-1}U_\parallel^\top\bar\Lambda,
    \qquad
    P_\perp^{\bar\Lambda}=I-P_\parallel^{\bar\Lambda}.
    \label{eq:appE:projdef}
\end{equation}
Since $\bar\Lambda\succ0$ and $U_\parallel$ has full column rank, $A\succ0$ (hence invertible, with $A^{-1}=A^{-\top}$).
Writing $P_i$ for either projector ($i\in\{\parallel,\perp\}$), the following hold:
\begin{enumerate}
    \item \emph{Idempotency:} $P_i^2=P_i$.
    \item \emph{Range and complementarity:} $P_\parallel^{\bar\Lambda}U_\parallel=U_\parallel$, $P_\perp^{\bar\Lambda}U_\parallel=0$, $\operatorname{range}P_\parallel^{\bar\Lambda}=\widehat{\mathcal S}_{\rm env}$, and $P_\parallel^{\bar\Lambda}+P_\perp^{\bar\Lambda}=I$.
    \item \emph{$\bar\Lambda$-self-adjointness:} $\bar\Lambda P_i=(\bar\Lambda P_i)^\top$, equivalently $P_i^\top\bar\Lambda=\bar\Lambda P_i$.
    \item \emph{Key quadratic identity:} $P_i^\top\bar\Lambda P_i=\bar\Lambda P_i$.
    \item \emph{Cross-orthogonality:} $P_i^\top\bar\Lambda P_j=\bar\Lambda P_iP_j=0$ for $i\neq j$.
\end{enumerate}
\end{proposition}

\begin{proof}
Write $P_\parallel$ for $P_\parallel^{\bar\Lambda}$; $A$ is SPD since $x^\top A x=(U_\parallel x)^\top\bar\Lambda(U_\parallel x)>0$ for $x\neq0$.
\emph{Idempotency} and \emph{range} follow from one Gram cancellation, $P_\parallel^2=U_\parallel A^{-1}(U_\parallel^\top\bar\Lambda U_\parallel)A^{-1}U_\parallel^\top\bar\Lambda=P_\parallel$ and $P_\parallel U_\parallel=U_\parallel A^{-1}A=U_\parallel$, with every $P_\parallel x$ a column combination of $U_\parallel$; the complementary identities for $P_\perp=I-P_\parallel$ are immediate.
\emph{Self-adjointness} is $\bar\Lambda P_\parallel=\bar\Lambda U_\parallel A^{-1}U_\parallel^\top\bar\Lambda=(\bar\Lambda P_\parallel)^\top$, using $A^{-1}=A^{-\top}$ and $\bar\Lambda^\top=\bar\Lambda$ (and likewise for $P_\perp$).
Combining self-adjointness with idempotency,
\begin{align}
    P_i^\top\bar\Lambda P_i &= (\bar\Lambda P_i)P_i = \bar\Lambda P_i^2 = \bar\Lambda P_i,
       &&\text{(self-adjointness + idempotency)} \label{eq:appE:keyid}\\
    P_i^\top\bar\Lambda P_j &= \bar\Lambda P_iP_j = 0 \quad (i\neq j),
       &&\text{($P_\parallel P_\perp=P_\parallel(I-P_\parallel)=0$)} \label{eq:appE:cross}
\end{align}
which are the key quadratic identity and cross-orthogonality.
\end{proof}

\noindent An immediate consequence, used in Appendix~\ref{app:c_psd}, is the positive semidefiniteness of $\bar\Lambda P_i$: for any $x\in\Rbb^6$,
\begin{equation}
    x^\top\bar\Lambda P_i x
    \stackrel{\eqref{eq:appE:keyid}}{=} x^\top P_i^\top\bar\Lambda P_i x
    = \|P_ix\|_{\bar\Lambda}^2\ge0,
    \label{eq:appE:psd_scalar}
\end{equation}
so $\bar\Lambda P_i=(\bar\Lambda P_i)^\top\succeq0$.

\paragraph{Dual (wrench) projector.}
The same operator splits wrenches by duality: under the power pairing $\langle F,V\rangle=F^\top V$ the wrench-side projector is $\bar P_i:=(P_i^{\bar\Lambda})^\top=\bar\Lambda\,P_i^{\bar\Lambda}\,\bar\Lambda^{-1}$ (self-adjointness, Prop.~\ref{prop:appE:proj}(3)).
By cross-orthogonality~\eqref{eq:appE:cross} a wrench in $\operatorname{range}\bar P_\perp$ does no power on a twist in $\widehat{\mathcal S}_{\rm env}$; for a true feasible twist the residual power is governed by the subspace-estimation error (SE metric, Appendix~\ref{app:exp_config}).

\subsection{Energy/Rayleigh-Based Gain Derivation}
\label{app:c_energy}

We assign separate elastic and damping costs to the $\bar\Lambda$-orthogonal error components, in the local coordinates $(e_T,e_V)$ of Eq.~\eqref{eq:tracking_errors} with the projectors held fixed:
\begin{equation}
    E_T(e_T)=\tfrac12\sum_{i\in\{\parallel,\perp\}}k_i\,\|P_i^{\bar\Lambda} e_T\|_{\bar\Lambda}^2,
    \qquad
    \mathcal R(e_V)=\tfrac12\sum_{i\in\{\parallel,\perp\}}d_i\,\|P_i^{\bar\Lambda} e_V\|_{\bar\Lambda}^2 .
    \label{eq:appE:energy}
\end{equation}
The key identity~\eqref{eq:appE:keyid} collapses each modal norm, $\|P_i^{\bar\Lambda} e\|_{\bar\Lambda}^2=e^\top\bar\Lambda P_i^{\bar\Lambda}e$, so $E_T,\mathcal R$ are quadratics whose Euclidean Hessians give the gains directly,
\begin{align}
    K_v
    &= \nabla^2_{e_T}E_T = \sum_i k_i\,\bar\Lambda P_i^{\bar\Lambda}
     = \bar\Lambda\bigl(k_\parallel P_\parallel^{\bar\Lambda}+ k_\perp P_\perp^{\bar\Lambda}\bigr),
    \notag\\
    D_v
    &= \nabla^2_{e_V}\mathcal R = \sum_i d_i\,\bar\Lambda P_i^{\bar\Lambda}
     = \bar\Lambda\bigl(d_\parallel P_\parallel^{\bar\Lambda}+ d_\perp P_\perp^{\bar\Lambda}\bigr),
    \label{eq:appE:gains}
\end{align}
the stiff-feasible/compliant-complement form of Eq.~\eqref{eq:decomposed_gains}.
Here $E_T$ is an exact conservative potential for the spring wrench $K_v e_T$ only in the first-order regime $\dot e_T=e_V$; we read it as the local elastic potential of the small-error passivity certificate (Appendix~\ref{app:e_passivity}), not a globally conservative $\SE(3)$ potential.

\subsection{Symmetry, PSD, and SPD}
\label{app:c_psd}

Each $\bar\Lambda P_i^{\bar\Lambda}$ is symmetric (Prop.~\ref{prop:appE:proj}(3)) and positive semidefinite (Eq.~\eqref{eq:appE:psd_scalar}), so as nonnegative combinations the gains $K_v=\sum_i k_i\bar\Lambda P_i^{\bar\Lambda}$ and $D_v=\sum_i d_i\bar\Lambda P_i^{\bar\Lambda}$ are symmetric PSD for all $k_i,d_i\ge0$.
They are symmetric positive definite (SPD) when the modal gains are strictly positive: since the cross term vanishes by~\eqref{eq:appE:cross}, the Pythagorean split gives, for $x\neq0$,
\begin{equation}
    x^\top K_v x
    = k_\parallel\|P_\parallel^{\bar\Lambda}x\|_{\bar\Lambda}^2 + k_\perp\|P_\perp^{\bar\Lambda}x\|_{\bar\Lambda}^2
    \ge \min(k_\parallel,k_\perp)\,\|x\|_{\bar\Lambda}^2 > 0 ,
\end{equation}
and likewise for $D_v$.
In our experiments $k_\perp,d_\perp$ are small but strictly positive (Appendix~\ref{app:iso_ours_oracle_gains}), so $K_v,D_v$ are SPD throughout the articulation phase.

\section{Closed-Loop Error Dynamics and Local Passivity}
\label{app:closed_loop}

We prove a local statement: near an operating point (fixed projectors, first-order error kinematics, frozen plant inertia), the decomposed feedback admits a quadratic storage function and a dissipation inequality.
Appendix~\ref{app:e_error_dyn} isolates the two ingredients (the $e_T$--$e_V$ kinematic relation and the frozen nominal model); Appendix~\ref{app:e_passivity} combines them.
The certificate is local and does not claim global stability for contact-rich manipulation.
Two inertias appear here: the true configuration-dependent plant inertia $\Lambda_b(q)$, and the constant surrogate $\bar\Lambda$ of Appendix~\ref{app:normalized_block_iso} carried by the gains.

\subsection{Local Error Dynamics}
\label{app:e_error_dyn}

\paragraph{Local error chart.}
Assume $T_b(t),T_d(t)$ smooth and within a chart where the principal logarithm of the relative rotation is single-valued.
With $R_e:=R_b^\top R_d$, $\rho:=R_b^\top(p_d-p_b)$, $\phi:=\log(R_e)^\vee$, so $e_T=(\phi,\rho)$, Eq.~\eqref{eq:tracking_errors} gives $e_\omega=R_e\omega_d-\omega_b$ and $e_v=\rho\times R_e\omega_d+R_e v_d-v_b$, so $e_V=(e_\omega,e_v)$.
The twist error is the right-trivialized relative-error velocity $(\dot T_eT_e^{-1})^\vee$ ($T_e:=T_b^{-1}T_d$) in current body coordinates, so $\dot R_eR_e^\top=(e_\omega)^\wedge$; with the left Jacobian $(J_l(\phi)\dot\phi)^\wedge=\dot R_eR_e^\top$ and $\dot\rho=e_v-\rho^\wedge e_\omega$, the exact local kinematics are
\begin{equation}
    \dot e_T = A(e_T)\,e_V,
    \qquad
    A(e_T)=
    \begin{bmatrix}
        J_l(\phi)^{-1} & 0\\
        -\rho^\wedge & I_3
    \end{bmatrix}.
    \label{eq:appF:etdot_exact}
\end{equation}
Since $J_l(\phi)^{-1}=I_3+O(\|\phi\|)$ and $\rho^\wedge=O(\|\rho\|)$, $A(0)=I_6$ and $A(e_T)e_V=e_V+O(\|e_T\|\,\|e_V\|)$, giving the first-order reconstruction
\begin{equation}
    \dot e_T = e_V
    \label{eq:appF:etdot}
\end{equation}
used in the nominal model ($J_l$ nonsingular near $\phi=0$, i.e.\ away from the rotation cut locus).

\paragraph{Frozen nominal model.}
Substituting $F_{\rm cmd}=F_{\rm ff}+K_v e_T+D_v e_V$ into the body-frame plant of Appendix~\ref{app:e_dynamics}, with $F_{\rm ff}$ supplying the modeled reference-acceleration, Coriolis/centrifugal, and gravity compensation, leaves the residual impedance model about the desired pose.
Freezing the plant inertia at $\Lambda_0:=\Lambda_b(q_0)\succ0$ (omitting a residual that collects the remaining $\Lambda_b(q)$ variation, higher-order pose terms, and modeling error) gives
\begin{equation}
    \Lambda_0\, \dot e_V + D_v\, e_V + K_v\, e_T = F_{\rm ext},
    \label{eq:appF:error_dyn}
\end{equation}
with $F_{\rm ext}$ the contact wrench in $\{b\}$ (sign as in Appendix~\ref{app:e_dynamics}) and $e_T,e_V$ as in Eq.~\eqref{eq:tracking_errors}.
Appendix~\ref{app:e_passivity} uses only~\eqref{eq:appF:error_dyn} and the first-order kinematics~\eqref{eq:appF:etdot}---the first-order frozen-inertia local model.

\subsection{Storage Function and Local Passivity}
\label{app:e_passivity}

With the first-order frozen-inertia local model in place (constant symmetric gains, $\dot e_T=e_V$, and the frozen-inertia impedance equation), the stiffness contribution to the storage derivative cancels exactly, leaving a pure dissipation term.

\begin{proposition}[Local passivity of the decomposed impedance loop]
\label{prop:appF:passivity}
Consider the first-order frozen-inertia local model of Appendix~\ref{app:e_error_dyn}, i.e.\ the impedance relation~\eqref{eq:appF:error_dyn} with the first-order kinematics~\eqref{eq:appF:etdot}.
Let
\begin{equation}
    K_v=\sum_{i\in\{\parallel,\perp\}}k_i\bar\Lambda P_i^{\bar\Lambda},
    \qquad
    D_v=\sum_{i\in\{\parallel,\perp\}}d_i\bar\Lambda P_i^{\bar\Lambda},
    \qquad
    k_i,d_i\ge0,
    \label{eq:appF:fixed_gains}
\end{equation}
with $U_\parallel$, $\bar\Lambda$, $k_i$, $d_i$, and the projectors held fixed.
Then the storage function
\begin{equation}
    \mathcal{V}(e_T, e_V) = \tfrac{1}{2}\, e_V^\top \Lambda_0\, e_V + E_T(e_T),
    \qquad
    E_T(e_T)=\tfrac12 e_T^\top K_v e_T,
    \label{eq:appF:storage}
\end{equation}
with $E_T$ the elastic potential of Appendix~\ref{app:c_energy}, satisfies along the nominal first-order model
\begin{equation}
    \dot{\mathcal{V}} = -\, e_V^\top D_v\, e_V + e_V^\top F_{\rm ext}\ \le\ e_V^\top F_{\rm ext}.
    \label{eq:appF:dotV}
\end{equation}
Hence the nominal first-order error system satisfies an error-coordinate dissipation inequality with supply $e_V^\top F_{\rm ext}$ (an error-coordinate supply rate, not physical contact-port passivity; see the scope note below).
In free motion ($F_{\rm ext}=0$), the origin of the linear frozen model is exponentially stable when $K_v\succ0$ and $D_v\succ0$, which holds when $k_\parallel,k_\perp,d_\parallel,d_\perp>0$.
With only semidefinite gains, the result is the dissipation inequality above, not convergence of all error coordinates.
\end{proposition}

\begin{proof}
Differentiating Eq.~\eqref{eq:appF:storage} in the first-order frozen local model and using $\nabla_{e_T}E_T=K_v e_T$ with $\dot e_T=e_V$ from Eq.~\eqref{eq:appF:etdot},
\begin{align}
    \dot{\mathcal V}
    &= e_V^\top \Lambda_0\, \dot e_V + e_T^\top K_v\, \dot e_T
       &&\text{(chain rule, $K_v$ fixed)}\nonumber\\
    &= e_V^\top \Lambda_0\, \dot e_V + e_V^\top K_v\, e_T
       &&\text{(first-order kinematics; $K_v=K_v^\top$)}\nonumber\\
    &= e_V^\top\!\left(F_{\rm ext} - D_v e_V - K_v e_T\right) + e_V^\top K_v\, e_T
       &&\text{(Eq.~\eqref{eq:appF:error_dyn})}\nonumber\\
    &= -\, e_V^\top D_v\, e_V + e_V^\top F_{\rm ext}.
       &&\text{(cancel $K_v$ terms)}
\end{align}
Because each $\bar\Lambda P_i^{\bar\Lambda}$ is symmetric PSD (Appendix~\ref{app:c_psd}), $D_v\succeq0$ for $d_i\ge0$, so $e_V^\top D_v e_V\ge0$ and Eq.~\eqref{eq:appF:dotV} follows.
If $K_v,D_v\succ0$, then $\mathcal V$ is positive definite and the unforced frozen model~\eqref{eq:appF:error_dyn} is a linear constant-coefficient system whose origin is exponentially stable (equivalently, the only invariant set in $\dot{\mathcal V}=0$ has $e_V=0$, forcing $K_v e_T=0$ and hence $e_T=0$).
With only semidefinite gains the same calculation proves the passivity inequality but need not force every error coordinate to converge.
\end{proof}

\paragraph{Scope.}
The certificate rests on (i) feedforward cancellation of the modeled $\hat\Lambda_b,\hat C_b,\hat g_b$ (Appendix~\ref{app:e_dynamics}), (ii) fixed projectors, so $K_v,D_v$ are constant and $E_T$ a genuine potential, (iii) the first-order kinematics~\eqref{eq:appF:etdot} in place of~\eqref{eq:appF:etdot_exact}, and (iv) the frozen inertia $\Lambda_b(q)\approx\Lambda_0$.
Collecting the omitted terms as a residual $r_{\rm dyn}$ adds $e_V^\top r_{\rm dyn}$ to Eq.~\eqref{eq:appF:dotV}: retaining the exact kinematics contributes $e_T^\top K_v(A(e_T)-I_6)e_V$ (higher order, not sign-definite), varying projectors or gains contributes $\tfrac12 e_T^\top\dot K_v e_T$, and an unfrozen plant contributes $\tfrac12 e_V^\top\dot\Lambda_b e_V$ with feedforward-cancellation residuals.
Under modeling error, contact-mode switches, large pose errors, projector updates, or appreciable $\Lambda_b(q)$ variation these need not preserve the inequality.
The certificate thus guarantees passivity-compatibility of the nominal fixed-projector, first-order frozen-inertia model only; it does not claim global stability, exact inertial decoupling, physical contact-port passivity under moving references, or passivity of arbitrary contact-rich manipulation.

\section{Gain Construction and Scaling}
\label{app:gain_parity}

We construct the scalar gains $k_\parallel,k_\perp,d_\parallel,d_\perp$ of Eq.~\eqref{eq:decomposed_gains} as \emph{formulas only}; numerical values ($k_0,r,\zeta,\bar I,\bar m$) are in Appendix~\ref{app:numerical_gains}.
The relevant metric is $\bar\Lambda$ of Appendix~\ref{app:normalized_block_iso}.
The $k_i,d_i$ are normalized modal gains; physical wrench-level stiffness and damping arise only after premultiplication by $\bar\Lambda$ (Appendix~\ref{app:normalized_block_iso}), which carries the inertia scale while only $\ell_\Lambda^2=\bar I/\bar m$ enters the directional split.

\subsection{Second-Order Gain Rule}
\label{app:second_order_gain_rule}

We fix the four scalars through interpretable second-order targets, a natural frequency $\omega_n$ and damping ratio $\zeta$, rather than through stiffness and damping directly.
Replacing the frozen plant inertia $\Lambda_0$ of Appendix~\ref{app:e_error_dyn} by the gain surrogate $\bar\Lambda$ gives the normalized design model
\begin{equation}
    \bar\Lambda\dot e_V+D_ve_V+K_ve_T=F_{\rm ext},
    \qquad
    \dot e_T=e_V ,
    \label{eq:appG:design_model}
\end{equation}
a scalar-gain convention only; the certificate of Appendix~\ref{app:e_passivity} retains the true $\Lambda_0$.
Premultiplying by $\bar\Lambda^{-1}$ and using $P_i^2=P_i$, $P_iP_j=0$ (Appendix~\ref{app:c_projectors}), each $\bar\Lambda$-orthogonal component evolves independently with modal gains $(k_\parallel,d_\parallel)$ on $\widehat{\mathcal S}_{\rm env}$ and $(k_\perp,d_\perp)$ on its complement, so in a $\bar\Lambda$-orthonormal coordinate $z$,
\begin{equation}
    \ddot z + d\,\dot z + k\,z = f_z ,
    \qquad (k,d)\in\{(k_\parallel,d_\parallel),(k_\perp,d_\perp)\}.
    \label{eq:appG:scalar_loop}
\end{equation}
Matching to $\ddot z+2\zeta\omega_n\dot z+\omega_n^2 z=f_z$ (so $\omega_n=\sqrt k$, $\zeta=d/(2\sqrt k)$) inverts to
\begin{equation}
    k=\omega_n^2,
    \qquad
    d=2\zeta\sqrt{k}=2\zeta\omega_n ,
    \label{eq:appG:gain_inversion}
\end{equation}
fixing each pair $(k_\parallel,d_\parallel)$, $(k_\perp,d_\perp)$ from a chosen $(\omega_n,\zeta)$.

\subsection{Iso, Ours, and Oracle Gains}
\label{app:iso_ours_oracle_gains}

We specify the three compared controllers (Section~\ref{sec:experiments}), all instantiating Eq.~\eqref{eq:decomposed_gains} via~\eqref{eq:appG:gain_inversion} and differing only in the feasible/constrained scalar gains and the subspace defining the projectors.
A single \emph{complementary reduction ratio} $r\in(0,1]$ scales the complement stiffness at fixed $\zeta$, with reference stiffness $k_0$ and $d_0=2\zeta\sqrt{k_0}$; $r=1$ gives Iso, $r<1$ gives Ours and Oracle.

\paragraph{Iso.}
Identical gains in every direction,
\begin{equation}
    k_\parallel=k_\perp=k_0,
    \qquad
    d_\parallel=d_\perp=d_0=2\zeta\sqrt{k_0} ,
    \label{eq:appG:iso}
\end{equation}
collapse the projector split (via $P_\parallel^{\bar\Lambda}+P_\perp^{\bar\Lambda}=I$) to $\bar\Lambda$-isotropic gains $K_v=k_0\bar\Lambda$, $D_v=d_0\bar\Lambda$.
Iso needs no subspace estimate and is the reference against which we measure constraint-violating load.

\paragraph{Ours.}
The feasible impedance matches Iso (Eq.~\eqref{eq:appG:iso}), $k_\parallel=k_0$ and $d_\parallel=d_0$ along $\widehat{\mathcal S}_{\rm env}$, while the complement stiffness is scaled, $k_\perp=r\,k_0$, and the complement damping is chosen to preserve $\zeta$:
\begin{align}
    d_\perp
    &= 2\zeta\sqrt{k_\perp}
    && \text{(inversion~\eqref{eq:appG:gain_inversion} at fixed $\zeta$)} \notag\\
    &= 2\zeta\sqrt{r\,k_0} = \sqrt{r}\,d_0 .
    && \text{($\sqrt r$ damping factor)}
    \label{eq:appG:ours_dperp}
\end{align}
The $\sqrt r$ factor is forced: holding $\zeta=d/(2\sqrt k)$ while reducing $k_\perp$ by $r$ drops $d_\perp$ by $\sqrt r$, not $r$.
For $r<1$ this gives $k_\parallel>k_\perp$ and $d_\parallel>d_\perp$---the stiff-feasible/compliant-complement structure of Eq.~\eqref{eq:decomposed_gains}---at a common damping ratio $\zeta$, hence comparable transient shape.
Values of $k_0,r,\zeta,\bar I,\bar m$ are in Appendix~\ref{app:numerical_gains}.

\paragraph{Oracle.}
Oracle (simulation only) uses the \emph{identical} gain rule as Ours (same $k_0,r,\zeta$ and Eq.~\eqref{eq:appG:ours_dperp}) but builds the projectors from the analytic reference subspace at each rollout's grasp pose, which is also the reference of the evaluation metrics (Appendix~\ref{app:impl}), rather than the estimated $\widehat{\mathcal S}_{\rm env}$.
Any Ours--Oracle gap therefore isolates subspace-estimation error, not the gain construction.

\section{Experimental Configuration}
\label{app:exp_config}

This section gives the configuration details for the experiments of Section~\ref{sec:experiments}: estimation and control hyperparameters and their sensitivity (Appendix~\ref{app:hyperparams}), numerical inertia and gain values (Appendix~\ref{app:numerical_gains}), and implementation---policy, platforms, demonstration data, sensing, and metrics (Appendix~\ref{app:impl}).

\subsection{Hyperparameters}
\label{app:hyperparams}

Table~\ref{tab:appH:common} lists the common configuration shared across all mechanisms.

\begin{table}[H]
\centering
\caption{\textbf{Common configuration.}
Estimation and control hyperparameters shared across all simulated and real-world mechanisms.
The reduction ratio $r$ and damping ratio $\zeta$ define the complement gains via Appendix~\ref{app:iso_ours_oracle_gains}; per-platform numerical gain values are in Table~\ref{tab:appH:numgains}.}
\label{tab:appH:common}
\vspace{0.25em}
\small
\begin{tabular}{p{0.12\linewidth} p{0.25\linewidth} p{0.50\linewidth}}
\toprule
Group & Symbol & Value \\
\midrule
\multirow{7}{*}{Estimation}
 & \raggedright Ensemble size $N$ & $128$ (simulation); $32$ (real) \\
 & \raggedright Trajectory horizon $H$ & $24$ \\
 & \raggedright Absolute threshold $v_{\rm abs}$ & $5$~mm/s (simulation); $3$~mm/s (real). \\
 & \raggedright Relative threshold $\kappa_{\rm rel}$ & $4$ (simulation); $10$ for revolute, prismatic, helical, and cylindrical objects and $4$ for the crank (real). \\
 & \raggedright Sampler at grasp closure & DDIM, $10$ steps (simulation); DDPM, $100$ steps (real). \\
 & \raggedright Characteristic-length clip $[\alpha_{\min},\alpha_{\max}]$ & $[0.03,\,0.5]$ \\
 & \raggedright Finite-difference timestep $\Delta t$ (s) & $0.1$ \\
\midrule
\multirow{2}{*}{Control}
 & \raggedright Modal gains $k_\parallel,k_\perp,d_\parallel,d_\perp$ & Per-platform; see Table~\ref{tab:appH:numgains}. \\
 & \raggedright Reduction ratio $r$ & $0.10$ (Ours, Oracle); $1.0$ (Iso). \\
\bottomrule
\end{tabular}
\end{table}

\paragraph{Dimension-selection thresholds.}
Both thresholds of Eq.~\eqref{eq:appD:dimsel} can be calibrated before deployment: $v_{\rm abs}$ from the residual speed that the intended policy configuration generates after training on constraint-satisfying trajectories of a known low-dimensional manifold (larger for noisier generation), and $\kappa_{\rm rel}$ from grasp-chain deviations in test demonstrations with the intended gripper--object setup (larger for more anisotropic compliance or clearance).
In simulation, the policies are trained on scripted demonstrations with accurate grasps, and the gripper and the object are rigid, so the off-manifold components come mainly from the model and sampler, and the absolute test dominates ($v_{\rm abs}=5$~mm/s, with $\kappa_{\rm rel}=4$): alone, it selects the correct dimension on all $398$ single-task trials in which the grasp succeeds (the grasp fails in one revolute and one planar trial).
On the real robot, finger deformation and handle clearance cause grasp-chain deviations that exceed the off-manifold components due to the model and sampler, which we keep small by sampling with $100$ DDPM steps, so the relative test dominates ($\kappa_{\rm rel}=10$, with $v_{\rm abs}=3$~mm/s): alone, it selects the correct dimension on all $50$ rollouts, with $\kappa_{\rm rel}=4$ for the crank to retain its weakly excited envelope direction.
With these values, the selected dimension is correct in all our rollouts with a completed grasp.

\paragraph{Sensitivity to the characteristic length.}
On one real rollout per object, holding the dimension fixed and scaling the clipped estimate $\alpha_0$ of Eq.~\eqref{eq:appD:alpha} by $\alpha/\alpha_0\in\{0.25,0.5,2,4\}$ moves the estimated subspace by at most $1.7^\circ$ (largest principal angle in the metric $\bar\Lambda$).

Figure~\ref{fig:appH:alpha_kv_heatmaps} shows the corresponding stiffness matrices ($N{=}32$, $H{=}24$): varying $\alpha$ can only rotate the stiff subspace, never rescale the gains, since $\{k_\parallel,k_\perp\},\{d_\parallel,d_\perp\}$ are the $\bar\Lambda$-eigenvalues of $K_v,D_v$ (Appendix~\ref{app:c_psd}).
The clip $[\alpha_{\min},\alpha_{\max}]$ only guards the PCA metric $G$ against degenerate angular--linear scalings (Appendix~\ref{app:pca_scaled}); on the real robot it is active only for the prismatic and helical objects.

\begin{figure}[t]
    \centering
    \includegraphics[width=\linewidth]{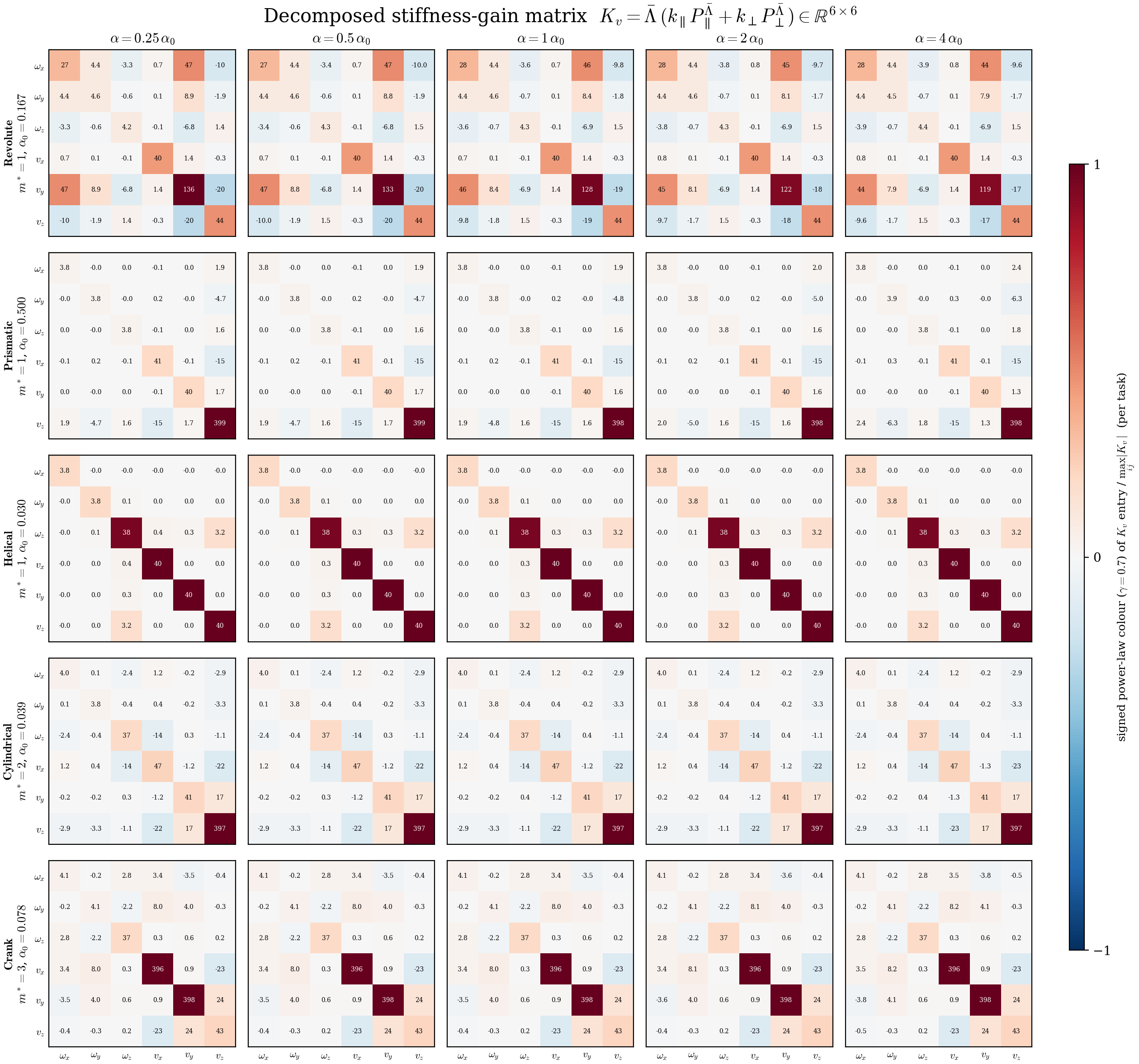}
    \caption{\textbf{Stiffness-gain matrices over the real FR3 $\alpha$ sweep.}
    Each panel shows the decomposed stiffness matrix $K_v=\bar\Lambda(k_\parallel P_\parallel^{\bar\Lambda} +k_\perp P_\perp^{\bar\Lambda})$ in body-twist order $(\omega_x,\omega_y,\omega_z,v_x,v_y,v_z)$.
    Rows correspond to real tasks, one rollout each (the one with the median SE), and columns to $\alpha/\alpha_0\in\{0.25,0.5,1,2,4\}$, with $\alpha_0$ the clipped estimate of Eq.~\eqref{eq:appD:alpha}.
    Printed numbers are the physical matrix entries; cell color encodes the signed, per-task-normalized value (blue $<0$, red $>0$).
    All five tasks vary little across the sweep.}
    \label{fig:appH:alpha_kv_heatmaps}
\end{figure}

\subsection{Numerical Gain Values}
\label{app:numerical_gains}

Table~\ref{tab:appH:numgains} reports the per-platform inertia surrogate ($\bar m,\bar I,\ell_\Lambda$, trace-averaged as in Appendix~\ref{app:normalized_block_iso}) and the scalar modal gains (Appendix~\ref{app:iso_ours_oracle_gains}).
\textbf{Iso} sets $r=1$; \textbf{Ours} and \textbf{Oracle} share these gains and differ only in using the estimated $\widehat{\mathcal S}_{\rm env}$ vs.\ the grasp-conditioned analytic reference.

\begin{table}[t]
\centering
\caption{\textbf{Numerical inertia and gain values.}
Per-platform inertia-metric parameters and the decomposed scalar gains under the construction of Appendix~\ref{app:iso_ours_oracle_gains}.
The inertia columns use kg, kg$\cdot$m$^2$, and m for $\bar m$, $\bar I$, and $\ell_\Lambda$, respectively.
The scalar gains $k_\parallel,k_\perp$ and $d_\parallel,d_\perp$ are normalized modal gains with units s$^{-2}$ and s$^{-1}$; physical wrench-level stiffness and damping units arise only after forming $K_v$ and $D_v$ with the normalized metric of Appendix~\ref{app:normalized_block_iso}.
The damping ratio $\zeta=\sqrt{\bar m/m_{\rm ref}}$ of this normalized design model makes every mode critically damped for the physical surrogate inertia; it is not a measured damping ratio of the robot--object system.
Iso uses $r=1$.}
\label{tab:appH:numgains}
\vspace{0.25em}
\small
\setlength{\tabcolsep}{4.5pt}
\begin{tabular}{l c c c c c c c c c c}
\toprule
Platform & $\bar m$ & $\bar I$ & $\ell_\Lambda$ & $k_0$ & $r$ & $\zeta$ & $k_\parallel$ & $k_\perp$ & $d_\parallel$ & $d_\perp$ \\
\midrule
Flying gripper (sim) & $0.586$ & $0.00622$ & $0.103$ & $250$ & $0.10$ & $0.766$ & $250$ & $25$ & $24.21$ & $7.66$ \\
FR3 (real)           & $7.70$ & $0.724$ & $0.307$ & $400$ & $0.10$ & $2.775$ & $400$ & $40$ & $111$ & $35.1$ \\
\bottomrule
\end{tabular}
\end{table}

\subsection{Implementation Details}
\label{app:impl}

\paragraph{Policy.}
All controllers share one generative backbone, a Diffusion Policy~\cite{chi2025diffusion} trained from scratch per object so its sampling diversity reflects that object's feasible motions; in simulation, we also train one multi-task policy on the demonstrations of all four objects, sampled equally and without a task identifier.
The training recipe is in Table~\ref{tab:appH:dp}; simulated policies use the scripted demonstrations of Table~\ref{tab:appH:simdata} and real policies the leader--follower teleoperation described below.

\begin{table}[t]
\centering
\caption{\textbf{Diffusion Policy training configuration.}
Common policy interface and training recipe used for the per-object checkpoints and the simulated multi-task checkpoint.}
\label{tab:appH:dp}
\vspace{0.25em}
\small
\begin{tabular}{p{0.31\linewidth} p{0.61\linewidth}}
\toprule
Setting & Value \\
\midrule
Backbone & Diffusion Policy~\cite{chi2025diffusion} \\
Observation history & $1$ step \\
Predicted / executed action horizon & $24$ / $16$ steps \\
Policy sampling period & $0.1$~s ($10$~Hz) \\
Observation & $224\times224$ wrist RGB, end-effector pose (position $+$ $6$-D
rotation), and gripper opening ($\Rbb^{10}$) \\
Action & Per-step $\Rbb^{10}$ relative pose chunk in the current EE frame:
pose ($9$) and gripper command ($1$) \\
Visual encoder & ResNet-18 + spatial softmax \\
Denoising steps (train / infer) & $100$ / $10$ DDIM (simulation), $100$ DDPM (real) \\
Optimizer & AdamW, $\beta=(0.95,0.999)$, weight decay $10^{-6}$ \\
Learning rate & $10^{-4}$, cosine decay \\
Batch size & $64$ \\
Training length / checkpoint & Sim: $500$ steps/epoch; the first checkpoint (saved every $50$ epochs) after the validation-loss minimum, epoch $250$/$150$/$100$/$150$ for revolute/cylindrical/planar/universal; multi-task: epoch $400$, the validation minimum. Real: $40$ steps/epoch, $600$ epochs; epoch-$150$ ($6$k steps) \\
\bottomrule
\end{tabular}
\end{table}

\paragraph{Platforms.}
Simulation uses NVIDIA Isaac Sim~\cite{NVIDIA_Isaac_Sim} with a flying gripper manipulating four articulated objects (revolute, cylindrical, planar, and universal joints; Figure~\ref{fig:sim_objects}).
Real-world experiments use a Franka Research~3 (FR3) torque-controlled arm on five physical objects (revolute, prismatic, helical, and cylindrical objects, plus a coffee-grinder crank).
The operational-space impedance law is run at a control frequency of $150$~Hz in simulation; the simulated flying gripper is gravity-free, and we use $V_d=0$ and $F_{\rm ff}=0$, omitting the Coriolis and reference-acceleration terms of Eq.~\eqref{eq:appA:Fff}.
On hardware, the FR3 controller receives dense desired-pose commands at $1$~kHz and the evaluation recorder logs aligned signals at $100$~Hz.

\paragraph{Simulation demonstration data.}
For each simulated object we collect $400$ demonstrations with a finite-state scripted expert: from a randomized object pose it closes the gripper and drives the object to a target configuration, computing handle poses by analytic forward kinematics.
This model information is withheld from the policy, the subspace estimator, and the controller at evaluation.
Table~\ref{tab:appH:simdata} lists the demonstrations and randomization.
After grasp closure the expert operates exactly each object's feasible coordinates---a single hinge rotation of $\pi$ for revolute; a $\pi$ rotation with $0.30$~m of axial travel for cylindrical; two in-plane translations and a half-turn for planar; and two $\pi/2$ revolute coordinates for the universal joint.
Replaying one fixed schedule would collapse the motion diversity the estimator depends on, so we randomize the operated-coordinate timing: a single-coordinate object follows one smooth zero-endpoint-velocity profile over $T=4.0$~s, while each coordinate of a multi-coordinate object (cylindrical, planar, universal) advances on its own profile with randomized onset and spread, yielding ensembles that range from near-synchronous to nearly sequential motion.

\begin{table}[H]
\centering
\caption{\textbf{Simulation policy data.}
Demonstration collection details for the four simulated objects (interface as in Table~\ref{tab:appH:dp}).
Stored actions are absolute end-effector poses, converted to relative pose chunks for training.
Placement is in the simulator world frame relative to the object origin $(0.5,0,0.5)$~m.}
\label{tab:appH:simdata}
\vspace{0.25em}
\small
\setlength{\tabcolsep}{4pt}
\begin{tabular}{l p{0.73\linewidth}}
\toprule
Item & Value \\
\midrule
Demonstrations & $400$ scripted rollouts per object \\
Object placement & Uniform $x\in[-0.15,0.15]$, $y\in[-0.20,0.20]$, $z\in[-0.10,0.10]$~m, yaw $\in[-15^\circ,15^\circ]$ \\
Gripper start & Fixed at $(-0.1,0,0.5)$~m before approach and grasp \\
Waypoint generation & Object progress mapped to handle poses by analytic forward kinematics, interpolated on $\SE(3)$ \\
\bottomrule
\end{tabular}
\end{table}

\paragraph{Real-world demonstration data.}
Real policies are trained from leader--follower teleoperation: the operator moves the leader arm and the follower FR3 logs its end-effector pose, wrist image, gripper state, and grasp commands at $10$~Hz.
We collect $20$ demonstrations each for the revolute, prismatic, and helical objects and $32$ each for the cylindrical joint and the coffee-grinder crank.
The policy interface matches simulation.
This follower-side pose, not the leader command, is the action the policy is trained to predict: in contact with the object it is physically shaped by the joint constraint and is thus (approximately) feasible, so the policy's sampled trajectories stay within the feasible subspace that the estimator of Section~\ref{sec:constraint_estimation} recovers.

\paragraph{Contact and force sensing.}
Projected wrench metrics use the gripper--object contact wrench at the end-effector frame in simulation, and a wrist-mounted ATI Axia80 six-axis force/torque sensor on the FR3.
These are used only for safety monitoring and evaluation, with no wrench feedback in the control loop (Appendix~\ref{app:e_dynamics}); \textbf{Iso} and \textbf{Ours} share the same safety thresholds and are scored over the same articulation range (below), keeping task progress comparable.

\paragraph{Evaluation metrics.}
Let $\mathcal S^{\rm ref}$ be the \emph{reference} feasible twist subspace, the analytic feasible subspace of the mechanism at the grasp pose of each rollout, in simulation and on the real robot, with basis $U^{\rm ref}=[(U^{\rm ref}_\omega)^\top,(U^{\rm ref}_v)^\top]^\top$.
We score wrench metrics \emph{blockwise} to avoid imposing an arbitrary length scale when combining force and moment into a single norm: with $Q_\omega^{\rm ref},Q_v^{\rm ref}$ orthonormal bases for the column spaces of $U^{\rm ref}_\omega,U^{\rm ref}_v$, the feasible/complement projectors are $P_{s,\parallel}=Q_s^{\rm ref}Q_s^{{\rm ref}\top}$ and $P_{s,\perp}=I_3-P_{s,\parallel}$ for $s\in\{\omega,v\}$.
These are marginal force/moment diagnostics, not a screw-coupled reciprocal-wrench norm.
Force is power-conjugate to the linear block and moment to the angular block under $F^\top V=f^\top v+\tau^\top\omega$, so the force uses the linear-block projectors $P_{v,\parallel/\perp}$ and the moment the angular-block projectors $P_{\omega,\parallel/\perp}$.
Splitting the measured external wrench into force $f_{\rm ext}$ and moment $\tau_{\rm ext}$ and writing $f_\perp=P_{v,\perp}f_{\rm ext}$, $\tau_\perp=P_{\omega,\perp}\tau_{\rm ext}$ (with $\parallel$ analogous), the primary metrics are the complement impulses and $95$th-percentile peaks
\begin{align}
\label{eq:appH:icf}
\text{ICF}&=\!\int\!\|f_\perp\|\,\mathrm dt, &
\text{ICM}&=\!\int\!\|\tau_\perp\|\,\mathrm dt, &&\text{(N$\cdot$s,\ N$\cdot$m$\cdot$s)}\\
\label{eq:appH:pcf}
\text{PCF}&=Q_{95}(\|f_\perp\|), &
\text{PCM}&=Q_{95}(\|\tau_\perp\|), &&\text{(N,\ N$\cdot$m)}
\end{align}
the precise blockwise form of the body metrics (Section~\ref{sec:experiments}).

As diagnostics we report two complementary fractions.
The complement \emph{load fractions} $\lambda_f,\lambda_m\in[0,1]$ give the share of the measured force and moment impulses in constrained directions, $\lambda_f=J_{f,\perp}/(J_{f,\parallel}+J_{f,\perp})$ and $\lambda_m=J_{\tau,\perp}/(J_{\tau,\parallel}+J_{\tau,\perp})$, where $J_{f,\parallel}=\int\|f_\parallel\|\,\mathrm dt$ and $J_{f,\perp}=\mathrm{ICF}$ (moment analogous).
The \emph{feedback feasible fractions} $\beta_f,\beta_m\in[0,1]$ instead give the reference-\emph{parallel} share of the commanded feedback wrench $(f_{\rm fb},\tau_{\rm fb})$, feedforward excluded: $\beta_f=J^{\rm fb}_{f,\parallel}/(J^{\rm fb}_{f,\parallel}+J^{\rm fb}_{f,\perp})$ with $J^{\rm fb}_{f,\parallel}=\int\|P_{v,\parallel}f_{\rm fb}\|\,\mathrm dt$, $J^{\rm fb}_{f,\perp}$ its complement, and $\beta_m$ analogous on $\tau_{\rm fb}$.
High $\beta$ means the feedback concentrates along feasible directions, opposite to $\lambda$'s complement numerator; a value is structural rather than informative when its block is zero-dimensional ($\beta{=}0$ when the feasible block is empty, $\beta{=}1$ when the complement block is empty).

\emph{Subspace Error} (SE, \textbf{Ours} only) is the largest principal angle, measured in the PCA metric $G$ (Euclidean angles after the transform $L_\alpha$ of Appendix~\ref{app:pca_scaled}, $G=L_\alpha^\top L_\alpha$), between $\widehat{\mathcal S}_{\rm env}$ and $\mathcal S^{\rm ref}$; in simulation and on the real robot alike, $L_\alpha$ uses a fixed evaluation length $\alpha^\star$ per object, set by the geometry of the analytic mechanism, rather than the estimator's own $\alpha$, so every rollout and controller is scored in the same metric.
\emph{Success Rate} (SR, simulation) counts trials that complete the common articulation range; a trial fails either because the policy does not complete the grasp or because the object does not reach the end of this range.
Every metric is computed over a common articulation range, identical for all controllers; on the real robot it runs from grasp completion until the object reaches a common absolute articulation target, so both controllers cover the same configuration change, and all $10$ rollouts per pair are kept without success filtering.

\paragraph{Evaluation sampling.}
Each simulation controller--object pair is evaluated over $100$ trials with the checkpoint of Table~\ref{tab:appH:dp}, for each single-task policy and for the multi-task policy; each real pair over $10$ rollouts with the epoch-$150$ checkpoint, i.e.\ $50$ trials per controller across the five objects.
In every trial \textbf{Ours} estimates its gains once at grasp closure, blends them in from the isotropic gains over $50$ control steps ($0.33$~s) in simulation, and holds them fixed.

\paragraph{End-to-end real-rollout visualization.}
Figure~\ref{fig:appH:real_ours_pipeline} traces the full Ours computation on representative real FR3 rollouts, from the grasp-time ensemble ($N{=}32$) to the final $6\times6$ impedance matrices $K_v,D_v$ in the gripper body frame; see the caption for the per-column construction.

\begin{figure}[p]
    \centering
    \includegraphics[width=\linewidth,height=0.78\textheight,keepaspectratio]{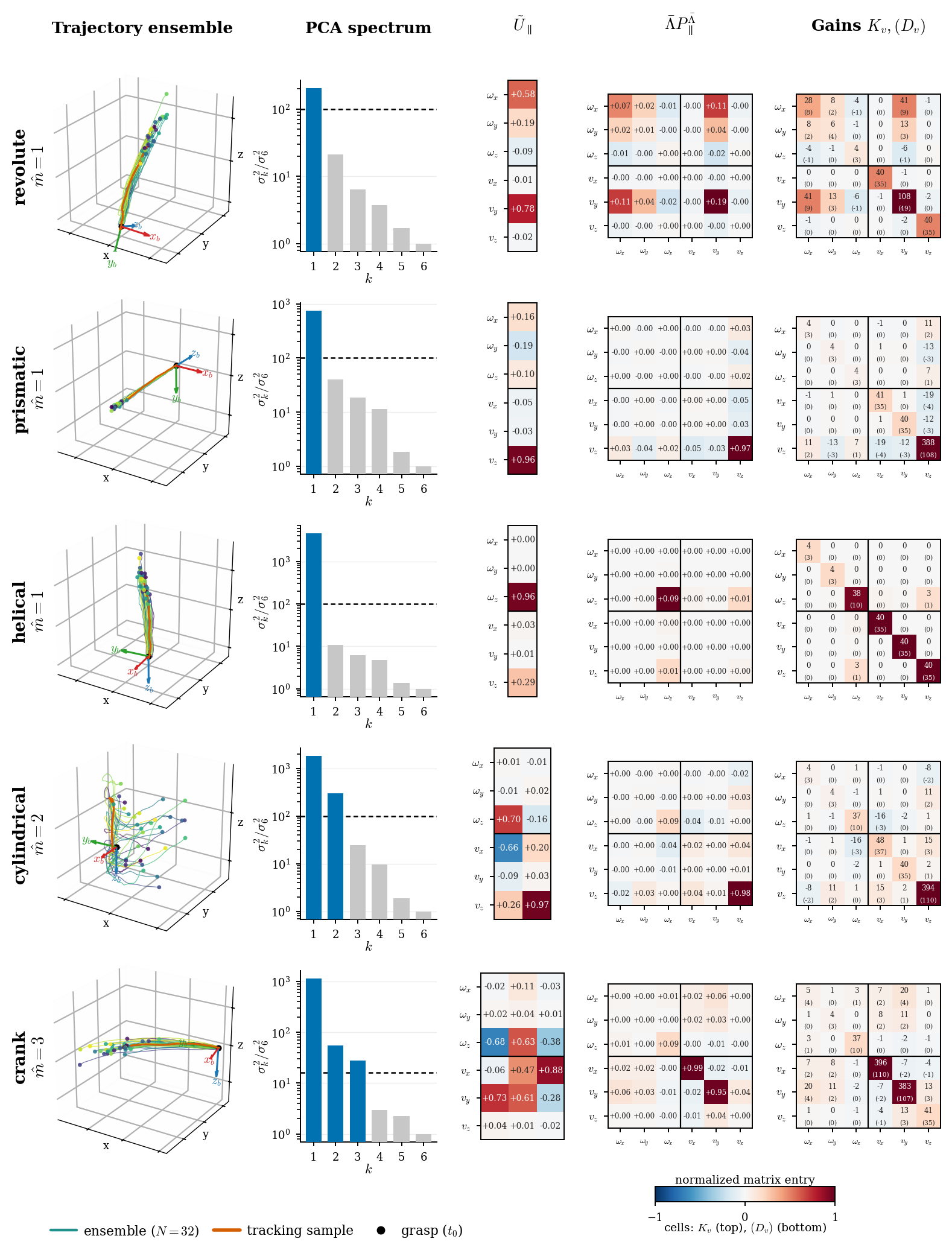}
    \caption{\textbf{From trajectory ensemble to constraint-decomposed impedance on the real FR3.}
    Each row is one real task, ordered by the estimated feasible dimension $\hat m$ (revolute/prismatic/helical $\hat m{=}1$, cylindrical $\hat m{=}2$, crank $\hat m{=}3$); all matrices are in the gripper body frame at the grasp pose.
    The columns trace the pipeline of Algorithm~\ref{alg:constraint_decomposed}: (1) the grasp-closure ensemble ($N=32$) recentered on the grasp point and colored by sample index, with the selected tracking trajectory highlighted and the body frame $(x_b,y_b,z_b)$ drawn as the red/green/blue triad; (2) the singular-value ratio $\widetilde\sigma_k^2/\widetilde\sigma_6^2$ $=\lambda_k/\lambda_6$ and the dimension gate $\kappa_{\rm rel}^2$ (dashed) selecting $\hat m$; (3) the retained PCA-scaled basis $\widetilde U_\parallel=L_\alpha U_\parallel$, rows ordered $(\omega_x,\omega_y,\omega_z,v_x,v_y,v_z)$; (4) the projector building block $\bar\Lambda P_\parallel^{\bar\Lambda}$ (Appendix~\ref{app:c_projectors}); and (5) the resulting stiffness $K_v=k_\perp\bar\Lambda+(k_\parallel-k_\perp)\bar\Lambda P_\parallel^{\bar\Lambda}$, with the damping entries $D_v$ in parentheses.
    Printed numbers are the physical entries; cell color encodes the signed, per-task-normalized value (blue $<0$, red $>0$; as in Fig.~\ref{fig:appH:alpha_kv_heatmaps}).}
    \label{fig:appH:real_ours_pipeline}
\end{figure}

\section{Simulation Experiments: Additional Results and Diagnostics}
\label{app:sim_extra}

This section expands Section~\ref{sec:exp_sim} with per-task force and moment diagnostics of the single-task and multi-task policies (Table~\ref{tab:appI:sim_force}).
Wrench and load-fraction entries are arithmetic means over the three-controller common-success set, so each row compares Iso, Ours, and Oracle on the same episodes.

\begin{table}[t]
\centering
\caption{\textbf{Simulation rollout diagnostics} (auxiliary to Table~\ref{tab:sim_main} and Table~\ref{tab:appI:multitask}).
Entries are Iso\,/\,Ours\,/\,Oracle.
$n_c$ is the three-controller common-success count for wrench metrics.
$\lambda_f,\lambda_m$ are the force/moment complement load fractions, $\beta_f,\beta_m$ the force/moment-side feedback feasible fractions (Appendix~\ref{app:impl}; $\beta_m{\equiv}1$ for the universal joint, whose angular complement is empty).}
\label{tab:appI:sim_force}
\vspace{0.25em}
\scriptsize
\setlength{\tabcolsep}{3.5pt}
\resizebox{\linewidth}{!}{%
\begin{tabular}{l c c c c c}
\toprule
Task & $n_c$ & $\lambda_f$ & $\lambda_m$ & $\beta_f$ & $\beta_m$\\
\midrule
\multicolumn{6}{l}{\textit{Single-task policies}}\\
Revolute    & \phantom{0}99 & .26 / .12 / .10 & .245 / .177 / .181 & .75 / .92 / .95 & .86 / .96 / .96\\
Cylindrical & \phantom{0}99 & .15 / .03 / .02 & .122 / .086 / .084 & .85 / .98 / .98 & .94 / .98 / .98\\
Planar      & \phantom{0}99 & .68 / .21 / .24 & .114 / .077 / .057 & .32 / .81 / .78 & .90 / .95 / .98\\
Universal   & 100 & .31 / .14 / .11 & .000 / .000 / .000 & .69 / .87 / .91 & 1.00 / 1.00 / 1.00\\
\midrule
\multicolumn{6}{l}{\textit{Multi-task policy}}\\
Revolute    & \phantom{0}99 & .27 / .13 / .11 & .242 / .197 / .190 & .74 / .92 / .94 & .86 / .96 / .95\\
Cylindrical & \phantom{0}95 & .13 / .04 / .02 & .119 / .076 / .077 & .87 / .96 / .98 & .93 / .98 / .98\\
Planar      & 100 & .68 / .22 / .24 & .112 / .078 / .056 & .32 / .81 / .78 & .90 / .95 / .98\\
Universal   & \phantom{0}99 & .29 / .14 / .11 & .000 / .000 / .000 & .72 / .88 / .92 & 1.00 / 1.00 / 1.00\\
\bottomrule
\end{tabular}%
}
\end{table}

For both the single-task and the multi-task policies, Ours lowers the force load fraction $\lambda_f$ and raises the force-side feedback feasible fraction $\beta_f$ on every object, close to the values of Oracle; the SE is largest for the configuration-dependent universal joint (Table~\ref{tab:sim_main}).

\paragraph{Multi-task policy.}
\label{app:sim_multitask}
We evaluate the multi-task policy of Appendix~\ref{app:impl}, trained on all four simulated objects, in the same way as the single-task policies.
Table~\ref{tab:appI:multitask} shows that the improvements of Table~\ref{tab:sim_main} are preserved: the force impulse of Ours is $0.14$--$0.37{\times}$ that of Iso, and Ours recovers at least $80\%$ of the Oracle's reduction relative to Iso in force impulse and peak force.

\begin{table}[t]
\centering
\caption{\textbf{Simulation results with one multi-task policy.}
Entries are Iso / Ours / Oracle.
In Table~\ref{tab:sim_main} and here, SR counts all trials, and the other entries are means over the trials in which all three controllers succeed; SE reports Ours only, and boldface marks the better of Iso and Ours.}
\label{tab:appI:multitask}
\vspace{0.25em}
\scriptsize
\setlength{\tabcolsep}{2.5pt}
\resizebox{\linewidth}{!}{%
\begin{tabular}{l c c c c c c}
\toprule
Task & SR (\%) & SE (\textdegree) & ICF (N$\cdot$s) & PCF (N) & ICM (N$\cdot$m$\cdot$s) & PCM (N$\cdot$m)\\
\midrule
Revolute & \phantom{0}99 / \phantom{0}99 / \phantom{0}99 & 1.71 & 7.75 / \textbf{2.87} / 2.55 & 3.55 / \textbf{1.76} / 1.81 & .31 / \textbf{.26} / .26 & \textbf{.21} / .22 / .22\\
Cylindrical & \phantom{0}96 / \phantom{0}\textbf{97} / \phantom{0}96 & 2.39 & 4.61 / \textbf{1.36} / 0.64 & 1.74 / \textbf{0.67} / 0.51 & .31 / \textbf{.20} / .21 & .15 / \textbf{.13} / .15\\
Planar & 100 / 100 / 100 & 2.76 & 14.4 / \textbf{2.02} / 2.21 & 3.85 / \textbf{1.13} / 1.13 & .73 / \textbf{.49} / .34 & .29 / \textbf{.24} / .22\\
Universal & \phantom{0}99 / \textbf{100} / 100 & 3.83 & 1.98 / \textbf{0.45} / 0.33 & 1.17 / \textbf{0.23} / 0.19 & -- / -- / -- & -- / -- / --\\
\bottomrule
\end{tabular}
}
\end{table}

\paragraph{Ensemble size and runtime.}
\label{app:sim_ensemble_size}
Table~\ref{tab:appI:ensemble_size} gives the full ensemble-size sweep of Section~\ref{sec:exp_analysis}, pooled over the four objects: the SE falls steeply up to $N{=}4$ and slowly afterwards, and the dimension rule selects the true dimension on every trial from $N{=}64$.
The runtime of $70.1$~ms in Section~\ref{sec:exp_analysis} is the median of the four per-object medians over $120$ runs ($10$ observations and $3$ seeds per object), and $76.5$~ms is the maximum over all runs.

\begin{table}[t]
\centering
\caption{\textbf{Ensemble-size sweep in simulation.} 95th-percentile SE with the dimension fixed to its true value, and the fraction of trials in which the dimension rule selects the true dimension, pooled over the four objects.}
\label{tab:appI:ensemble_size}
\vspace{0.25em}
\small
\begin{tabular}{r c c}
\toprule
$N$ & SE, 95th percentile (\textdegree) & Correct dimension (\%) \\
\midrule
$1$   & $78.95$ & $64.7$ \\
$2$   & $23.15$ & $84.2$ \\
$4$   & $8.59$  & $92.8$ \\
$8$   & $7.23$  & $97.7$ \\
$16$  & $6.92$  & $99.4$ \\
$32$  & $6.79$  & $99.8$ \\
$64$  & $6.68$  & $100$ \\
$128$ & $6.61$  & $100$ \\
\bottomrule
\end{tabular}
\end{table}

\paragraph{Sample diversity.}
\label{app:sim_diversity}
Table~\ref{tab:appI:diversity} gives the full results of the sample-diversity study of Section~\ref{sec:exp_analysis}, in which the most and least diverse subsets are chosen among $4012$ candidate $k$-sample subsets of each ensemble.
The most diverse subsets are markedly more accurate for $k\le16$; from $k{=}32$, the two selections differ by at most $0.1^\circ$ and approach the full-ensemble value.

\begin{table}[t]
\centering
\caption{\textbf{Sample-diversity study in simulation.}
Subspaces re-estimated from the $k$ most (DTW-max) or least (DTW-min) diverse samples of each $128$-sample ensemble; the last column uses all $128$ samples.}
\label{tab:appI:diversity}
\vspace{0.25em}
\small
\begin{tabular}{l l c c c c c c c}
\toprule
Statistic & Selection & $k{=}2$ & $4$ & $8$ & $16$ & $32$ & $64$ & $128$ \\
\midrule
Mean of per-object & DTW-max & $5.91$ & $3.32$ & $2.74$ & $2.61$ & $2.48$ & $2.49$ & $2.34$ \\
median SE (\textdegree) & DTW-min & $19.73$ & $9.93$ & $4.45$ & $3.43$ & $2.58$ & $2.43$ & $2.34$ \\
\midrule
Grasps with & DTW-max & $16.0$ & $8.8$ & $2.8$ & $2.3$ & $1.8$ & $1.0$ & $1.0$ \\
SE $>10^\circ$ (\%) & DTW-min & $40.5$ & $28.0$ & $13.8$ & $9.5$ & $1.5$ & $1.5$ & $1.0$ \\
\bottomrule
\end{tabular}
\end{table}

\paragraph{Robustness to noise.}
\label{app:sim_noise}
Table~\ref{tab:appI:noise} gives the full results of the noise studies of Section~\ref{sec:exp_analysis}; the added noise is per-step Gaussian.
Orientation noise in the demonstrations still raises SE to $10.59^\circ$: a policy trained on corrupted demonstrations can keep its success rate and validation loss and still yield wrong compliant directions.
With $2$ DDIM steps, the estimate collapses.

\begin{table}[t]
\centering
\caption{\textbf{Robustness to noise in simulation.}
Pooled median SE (\textdegree) over the four objects at $N{=}128$.
(a)~Noise of the same amplitude added directly to the sampled trajectories or to the demonstrations before training.
(b)~Image perturbations of the observation: Gaussian noise ($\sigma{=}32/255$), Gaussian blur ($\sigma{=}2$~px), and random shift ($\sigma{=}4$~px).
(c)~Number of DDIM steps at inference.}
\label{tab:appI:noise}
\vspace{0.25em}
\footnotesize
\setlength{\tabcolsep}{4pt}
\begin{minipage}[t]{0.42\linewidth}
\centering
(a) Training data\\[0.3em]
\begin{tabular}{l c c c}
\toprule
Noise & Direct & Trained & Ratio \\
\midrule
None & $2.14$ & $2.02$ & -- \\
Position $5$~mm & $11.97$ & $3.10$ & $0.26$ \\
Orientation $5^\circ$ & $16.05$ & $10.59$ & $0.66$ \\
\bottomrule
\end{tabular}
\end{minipage}\hfill
\begin{minipage}[t]{0.30\linewidth}
\centering
(b) Observations\\[0.3em]
\begin{tabular}{l c}
\toprule
Perturbation & SE \\
\midrule
None & $2.10$ \\
Gaussian noise & $2.21$ \\
Gaussian blur & $2.12$ \\
Random shift & $2.14$ \\
\bottomrule
\end{tabular}
\end{minipage}\hfill
\begin{minipage}[t]{0.24\linewidth}
\centering
(c) Inference\\[0.3em]
\begin{tabular}{r c}
\toprule
DDIM steps & SE \\
\midrule
$100$ & $2.16$ \\
$10$ & $2.14$ \\
$8$ & $2.12$ \\
$6$ & $2.14$ \\
$4$ & $2.26$ \\
$2$ & $87.38$ \\
\bottomrule
\end{tabular}
\end{minipage}
\end{table}

\section{Real-World Experiments: Additional Results and Diagnostics}
\label{app:real_extra}

This section expands Section~\ref{sec:exp_real} with diagnostics auxiliary to body Table~\ref{tab:real_main}: Table~\ref{tab:appJ:real_aux} reports per-task load-fraction and feedback diagnostics, all scored against the grasp-conditioned analytic reference of Appendix~\ref{app:impl}.

\begin{table}[t]
\centering
\caption{\textbf{Real-world auxiliary diagnostics.}
Entries are Iso / Ours (Ours/Iso) arithmetic means over the same $10$ rollouts per task as Table~\ref{tab:real_main}.
$\lambda_f,\lambda_m$ are the measured force/moment load fractions and $\beta_f,\beta_m$ the feedback feasible fractions of Appendix~\ref{app:impl}.
The rotational block of prismatic is structurally complementary ($\lambda_m{=}1$, $\beta_m{=}0$).}
\label{tab:appJ:real_aux}
\vspace{0.25em}
\scriptsize
\setlength{\tabcolsep}{4.5pt}
\resizebox{\linewidth}{!}{%
\begin{tabular}{l c c c c}
\toprule
Task & $\lambda_f$ & $\lambda_m$ & $\beta_f$ & $\beta_m$\\
\midrule
Revolute    & .766 / .672 (.88)  & .734 / .522 (.71) & .738 / .843 (1.14) & .717 / .848 (1.18) \\
Prismatic   & .362 / .321 (.89)  & 1.000 / 1.000 (1.00) & .848 / .918 (1.08) & .000 / .000 (--) \\
Helical     & .412 / .424 (1.03) & .724 / .569 (.79) & .187 / .332 (1.78) & .751 / .890 (1.18) \\
Cylindrical & .244 / .126 (.52)  & .536 / .421 (.79) & .708 / .894 (1.26) & .572 / .680 (1.19) \\
Crank       & .425 / .469 (1.10) & .926 / .850 (.92) & .841 / .892 (1.06) & .261 / .550 (2.11) \\
\bottomrule
\end{tabular}%
}
\vspace{-0.5em}
\end{table}

Averaged over the objects, Ours lowers the complement load fractions ($\lambda_f$ $0.91\times$, $\lambda_m$ $0.81\times$) and raises the feedback feasible fractions ($\beta_f$ $1.17\times$, $\beta_m$ $1.29\times$), with prismatic excluded for $\lambda_m$ and $\beta_m$.
The force load fraction rises slightly on the helical object ($1.03\times$) and on the crank ($1.10\times$), whose feedback moment is nonetheless much better aligned with the feasible directions ($\beta_m$ $2.11\times$).

\end{document}